\documentclass{article}
\usepackage[affil-it]{authblk}
\usepackage[utf8]{inputenc}
\usepackage[a-1b]{pdfx}
\usepackage{amsthm}
\usepackage{hyperref}
\usepackage{booktabs}
\usepackage{adjustbox}
\newtheorem{theorem}{Theorem}

\usepackage[margin=1in]{geometry}

\usepackage{graphics,multicol}
\usepackage{natbib}
\usepackage{epsfig}
\usepackage{wrapfig}
\usepackage{tikz}
\usetikzlibrary{positioning,arrows.meta}
\usepackage{float}
\usepackage{bbm}
\usepackage{subcaption}
\usepackage{amsmath,amsfonts}
\usepackage{amssymb}
\usepackage{rotating}
\usepackage{pdfpages}
\usepackage{color}
\usepackage{epic}
\usepackage{eepic}
\usepackage[english]{babel}
\usepackage{enumerate}
\usepackage{url}
\usepackage{xcolor}
\usepackage{algorithm}
\usepackage{algpseudocode}
\usepackage{tabularx}
\usepackage{array}
\usepackage{booktabs}
\usepackage{pifont}

\newcommand{\cmark}{\ding{51}}

\newtheorem{corollary}[theorem]{Corollary}

\newtheorem{proposition}[theorem]{Proposition}

\theoremstyle{definition}
\newtheorem{definition}[theorem]{Definition}

\title {Structured Learning on Mapper Representations}

\author[1]{George Babus}
\author[1]{Farzana Nasrin}

\affil[1]{Department of Mathematics,
       University of Tennessee,
       Knoxville, TN -37916, USA}

\date{}

\begin{document}

\maketitle
\begin{abstract}
Modern machine learning (ML) methods are highly effective for prediction tasks, but many commonly used representations reduce complex data to fixed dimensional embeddings that may suppress multiscale structural organization. The Mapper algorithm from topological data analysis (TDA) provides a different perspective by decomposing data into overlapping local regions connected through a nerve construction, producing a structured representation that captures geometric organization, local statistical behavior, and relational connectivity simultaneously.
In this work, we develop a framework for learning over Mapper induced structured representations. Rather than treating Mapper as a preprocessing step that produces a graph for downstream learning, we treat the full Mapper construction as part of the representation itself. 
We study mathematical properties of these representations, including invariance under relabeling, a distance functional on the space of Mapper representations, structural complexity of multiscale decompositions, and learning oriented stability under representation perturbations. 
Experiments on time series and graph classification datasets validate the proposed framework through controlled studies of representation ablation, Mapper parameter sensitivity, and the geometry of the induced representation space.
Together, these results demonstrate how the proposed mathematical framework enables systematic comparison, interpretation, and analysis of Mapper representations, providing practical tools for studying representation geometry, structural complexity, and learning stability in learning tasks.
\end{abstract}

\textbf{keywords} Topological data analysis, Mapper, graph neural networks, statistical learning theory, structured prediction

\maketitle


\section{Introduction}

\subsection{Motivation}

    Machine learning (ML) relies fundamentally on representations. Before prediction is performed, complex datasets are transformed into objects that encode information about the underlying data and serve as the input to learning algorithms \cite{bengio2013representation, goodfellow2016deep, hamilton2017representation, lecun2015deep}. Depending on the application, these representations may take the form of Euclidean embeddings, graphs, sequences, or other structured objects. The choice of representation determines what information is retained, how different datasets are compared, and ultimately what can be learned from the data. Consequently, the mathematical study of representation spaces has become an increasingly important direction in modern ML, complementing advances in prediction algorithms themselves.
    
    Most widely used representations are constructed to facilitate optimization in vector spaces. Topological Data Analysis (TDA) provides a complementary paradigm for representing complex data through geometric and topological structure \cite{Carlsson2009,edelsbrunner2010computational}. Among its most widely used tools, the Mapper algorithm constructs an interpretable multiscale decomposition of data through filtering, overlapping covers, clustering, and nerve construction \cite{singh2007topological}. Mapper has been successfully applied to exploratory analysis and visualization in biology, medicine, genomics, and other scientific domains \cite{nicolau2011topology,lum2013extracting, Amezquita2023}, while its statistical properties have been investigated in studies of consistency, stability, and parameter selection \cite{carriere2018statistical}. 

    Despite its broad success, Mapper is almost always viewed as an intermediate computational tool rather than as a representation in its own right. In many learning pipelines, the output of Mapper is reduced to a graph valued embedding or further summarized into fixed dimensional feature vectors before downstream prediction \cite{cyranka2019mapper,deepgraphmapper,curry2023topologically,comparingmappergraphs}. Consequently, learning is typically performed on derived graph or vector representations rather than on the complete Mapper construction itself. However, the Mapper algorithm naturally produces a much richer object consisting of interacting components, including the lens induced cover, pullback decomposition, refined local clusters, overlap relationships, nerve connectivity, and associated local attributes. We regard these components collectively as defining the representation itself. This perspective enables individual components of the Mapper construction to be studied as constituents of the representation itself, allowing their geometric properties, structural complexity, and predictive contributions to be analyzed within a unified mathematical framework.
    

    In this work, we formalize the complete Mapper construction as the representation itself. Instead of identifying a dataset with its nerve graph, we represent it by the full multiscale decomposition generated during the Mapper pipeline. Learning is then formulated as the composition of a representation map with predictors acting on the resulting representation space. This viewpoint separates representation construction from prediction and allows mathematical properties of the representation to be studied independently of any particular learning architecture. 
    Although the present paper illustrates the approach using graph neural networks operating on attributed Mapper graphs, the underlying representation is the complete Mapper decomposition and is not tied to a specific learning architecture.

    Once Mapper is regarded as a representation space rather than merely a graph construction, several mathematical questions arise naturally. First, different representations should be compared independently of arbitrary relabeling of clusters, motivating notions of equivalence, distance, and kernels on representation space. Second, because Mapper representations possess multiscale overlap structure, representation complexity depends on structural characteristics beyond graph size alone. Finally, predictive learning should be robust to perturbations of the representation, motivating stability analyses that relate structural changes to changes in prediction.

Guided by this representation centric viewpoint, we develop a mathematical framework for studying Mapper representations from both geometric and statistical perspectives. The framework has two complementary components. First, we introduce a structural geometry on the space of Mapper representations through permutation invariant equivalence relations, distance functional, and kernels that enable quantitative comparison between Mapper objects. Second, we develop theoretical notions of representation complexity together with stability results for graph based predictors acting on Mapper derived representations. We then validate these mathematical developments through a comprehensive empirical study that investigates the predictive contribution of different representation components, the influence of Mapper construction parameters, and the structural organization of Mapper representations across time series and graph classification tasks.
    

    \subsection{Related Work}
    
    Topological Data Analysis (TDA) has emerged as a powerful framework for
    studying the shape of data \cite{Carlsson2009, Wasserman2018}.
    Persistent homology provides stable summaries of topological features
    across scales \cite{edelsbrunner2010computational, ChazalMichel2021},
    while the Mapper algorithm produces graph based representations that capture
    multiscale structure \cite{singh2007topological}. Mapper has been widely
    used for exploratory analysis and scientific discovery in complex datasets,
    including biological and medical applications \cite{nicolau2011topology,
    lum2013extracting}. Statistical properties of Mapper and parameter selection
    have been studied in several works, including consistency and stability
    analyses \cite{carriere2018statistical}.
    
    Graph representation learning has developed largely independently,
    focusing on embedding nodes or graphs into Euclidean spaces.
    Methods such as spectral embeddings \cite{belkin2003laplacian},
    random-walk-based embeddings \cite{perozzi2014deepwalk, grover2016node2vec},
    and graph neural networks \cite{kipf2017gcn, hamilton2017representation}
    aim to preserve structural information while enabling efficient learning.
    However, these approaches typically operate within vector spaces and do
    not explicitly incorporate higher level topological structure.
    
    Several works have combined topology with machine learning by converting
    topological summaries into feature vectors. For example, persistence
    diagrams are often vectorized through kernel methods or learned embeddings
    \cite{reininghaus2015stable, adams2017persistence, kusano2016persistence}.
    Related efforts have explored the use of Mapper representations within
    learning pipelines, including Mapper based classifiers
    \cite{cyranka2019mapper}, hierarchical representations using Mapper with
    graph neural networks \cite{deepgraphmapper}, and topologically attributed
    graphs for shape classification \cite{curry2023topologically}. 
    Other
    work studies methods for comparing Mapper graphs in machine learning
    settings, such as analyzing neural network activations
    \cite{comparingmappergraphs}. 
    Despite their practical success, many of these approaches ultimately
    reduce topological objects to Euclidean representations prior to
    learning, thereby flattening structural information encoded in the
    Mapper graph into fixed dimensional feature vectors.
    In contrast, the present work develops a framework in which the
    topological representation itself is the object of learning. This
    perspective is related to structured prediction and learning over
    combinatorial objects \cite{bakir2007predicting, nowozin2011structured},
    but it requires new tools to handle invariances and complexity specific
    to Mapper constructions. To our knowledge, a unified mathematical framework for the representation, comparison, complexity, and statistical analysis of Mapper derived representations has not been previously developed.

    \subsection{Contributions}

The main contributions of this work are summarized below.

\begin{itemize}

\item \textbf{Learning over Mapper representations.}
We develop a learning framework that treats the complete Mapper representation, rather than a vector embedding or graph alone, as the fundamental object of analysis.

\item \textbf{Structural geometry of Mapper representations.}
We introduce permutation-invariant equivalence relations, structural distance functionals, and kernels that enable principled comparison between Mapper representations.

\item \textbf{Complexity and stability analysis.}
We develop notions of structural complexity for Mapper representations and establish stability properties for graph-based predictors acting on these representations.

\item \textbf{Experimental characterization of Mapper representations.}
Through experiments on time-series and graph classification datasets, we investigate (i) the predictive contribution of different representation components, (ii) the influence of Mapper construction parameters on learning behavior, and (iii) the structural organization of Mapper representations under the proposed distance.
\end{itemize}
\begin{figure}[h!]
    \centering
    \includegraphics[trim=  0cm 8cm 1cm 2cm,
    clip,
    width=5.5in,height= 2.15in]{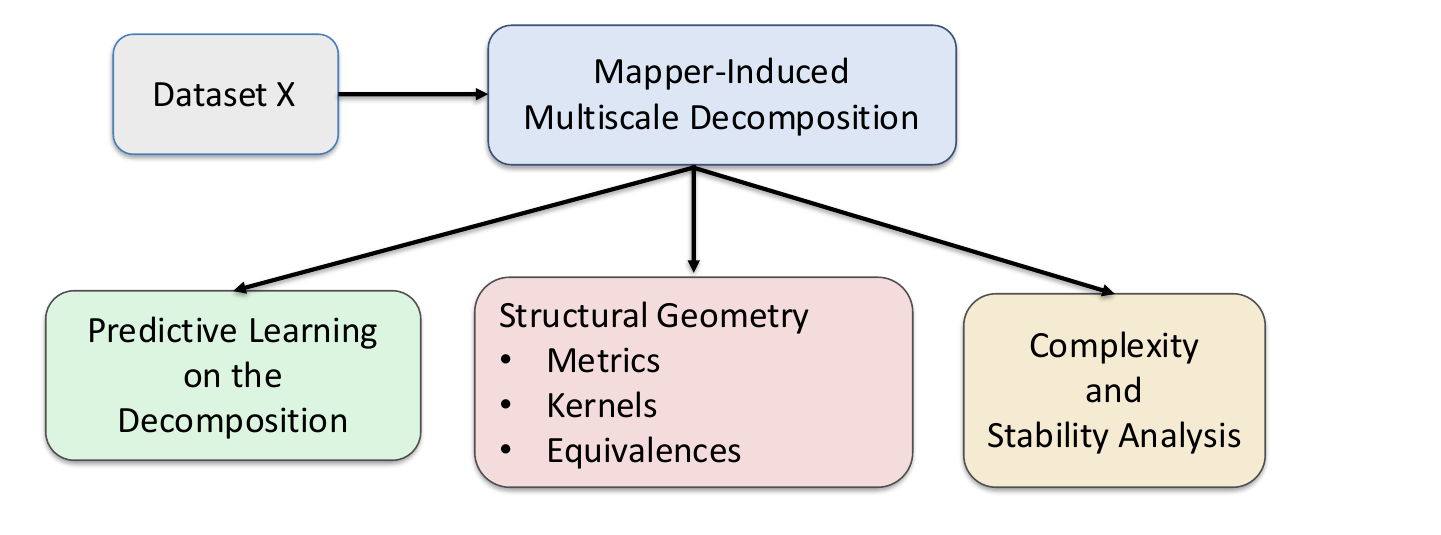}
    \caption{Overview of the proposed 
    framework. 
    A dataset is transformed into a structured multiscale
    decomposition through the Mapper construction. The resulting
    representation serves as the foundation for three interconnected
    components of the framework.
    }
    \label{fig:framework}
    \end{figure}

            The overall organization of the proposed framework is summarized in
    Figure~\ref{fig:framework}. 
     The remainder of the paper is organized as follows. Section~ \ref{sec:prelim} reviews
the Mapper construction and introduces the learning setting underlying
the proposed framework. Section~ \ref{sec:framework} develops the notion of
Mapper induced representations and formulates predictive learning
directly on these structured representation spaces. Section~ \ref{sec:geometry}
establishes the mathematical geometry of Mapper representations,
including equivalence relations, structural distances, kernels, and
multiscale refinement. Section~ \ref{sec:complexity_stability} develops theoretical notions of
representation complexity together with stability results for
graph based predictors acting on projected Mapper representations.
Finally, Section~\ref{sec:exp} presents experiments on time series and graph
classification datasets to investigate predictive contribution,
representation complexity, and robustness under structural
perturbations.


\section{Preliminaries} \label{sec:prelim}

This section introduces the basic notions underlying our framework.
We first describe the
general learning setup and then review the Mapper construction.

\subsection{Data domains and notation}

To establish notation used consistently throughout the paper, we briefly describe the general learning setting in which our framework is developed. The formulation is intentionally broad and is not specific to any particular learning algorithm or data modality.
Let $\mathcal{X}$ denote a domain of datasets. An element
$X \in \mathcal{X}$ is a finite collection of objects equipped with
sufficient structure to compute distances, similarities, or features.
This formulation is intentionally general and allows the framework
to apply to a variety of data types, including point clouds, images or signals, time series observations, and vertex sets of graphs with associated features. 
In a supervised learning setting, each dataset
$X \in \mathcal{X}$ is associated with an output
$y \in \mathcal{Y}$, where $\mathcal{Y}$ denotes a target
space such as class labels or real valued responses.
Given training samples
$
\{(X_i,y_i)\}_{i=1}^n \subset \mathcal{X}\times\mathcal{Y},
$
the goal is to learn a predictor
$
h:\mathcal{X}\rightarrow\mathcal{Y}
$
that generalizes to unseen datasets.


\subsection{The Mapper construction}

The Mapper algorithm is a method for summarizing the global
structure of data through a graph representation.
It proceeds by combining three ingredients:
a low dimensional representation of the data,
a cover of that representation, and clustering within
overlapping regions of the data.
The first step of the Mapper construction assigns
coordinates to each data point.

\begin{definition}[Lens function]
Given a dataset $X$, a \emph{lens} is a function
\[
f:X\rightarrow\mathbb{R}^d
\]
that maps each data point to a coordinate vector.
\end{definition}
The lens may arise from domain specific features,
dimensionality reduction methods (e.g., PCA),
or learned representations.
The next step partitions the image of the dataset
under the lens using overlapping regions.

\begin{definition}[Cover]
Let $Y\subset\mathbb{R}^d$.
A collection of sets
\[
\mathcal{U}=\{U_1,\dots,U_K\}
\]
is a \emph{cover} of $Y$ if $
Y\subseteq\bigcup_{k=1}^{K}U_k$.
\end{definition}

In practice, covers are often constructed from overlapping
intervals or rectangular regions in $\mathbb{R}^d$.
The cover defined in the lens space is then transferred
back to the original dataset. This step is referred as pullback of the cover.

\begin{definition}[Pullback cover]
Let $\mathcal{U}$ be a cover of $f(X)$.
The pullback cover of $X$ is
\[
f^{-1}(\mathcal{U})
=
\{f^{-1}(U_k):U_k\in\mathcal{U}\}.
\]
\end{definition}

Each pullback set contains the data points whose lens
values lie inside a particular region of the cover.
Since each pullback set may still contain heterogeneous
points, Mapper refines these sets using a clustering
procedure. 

\begin{definition}[Refined pullback cover]
Let $\mathcal{C}$ be a clustering algorithm.
The refined pullback cover is
\[
\overline{\mathcal{U}}
=
\bigcup_{k=1}^{K}\mathcal{C}(f^{-1}(U_k)),
\]
the collection of clusters obtained by clustering
each pullback set.
\end{definition}

Because the original cover regions overlap,
clusters obtained from different pullback sets
may also overlap. Finally, Mapper constructs a graph structure derived from the nerve construction.
The connectivity structure among these clusters
is encoded using the nerve construction.

\begin{definition}[Nerve]
Given a collection of sets $\mathcal{A}$,
the \emph{nerve} $\mathcal{N}(\mathcal{A})$
is the simplicial complex whose vertices
correspond to elements of $\mathcal{A}$ and
whose simplices correspond to nonempty
intersections among those sets.
\end{definition}


\begin{definition}[1-skeleton]
Let $K$ be a simplicial complex. The \emph{1-skeleton} of $K$,
denoted by $\mathrm{sk}_1(K)$, is the graph consisting of the
vertices and edges of $K$, with all higher-dimensional simplices
discarded.
\end{definition}

\noindent Combining the steps above yields the Mapper
representation of the dataset.

\begin{definition}[Mapper object]
Given a dataset $X$, lens $f$, cover $\mathcal{U}$
of $f(X)$, and clustering procedure $\mathcal{C}$,
the \emph{Mapper object} is the graph
\[
M=\mathrm{sk}_1\big(\mathcal{N}(\overline{\mathcal{U}})\big),
\]
the 1-skeleton of the nerve of the refined
pullback cover.
\end{definition}

Nodes of the Mapper graph correspond to clusters of data points,
while edges indicate overlapping clusters. Additional attributes,
such as cluster sizes or summary statistics, may also be attached
to nodes and edges. Taken together, the Mapper pipeline constructs a multiscale
representation. The resulting Mapper
graph therefore captures both local organization and global
connectivity patterns within the dataset. Figure~\ref{fig:mapper} summarizes the overall
Mapper construction used throughout this work.





\begin{figure}[t!]
	\centering
	\includegraphics[width=1.1\linewidth]{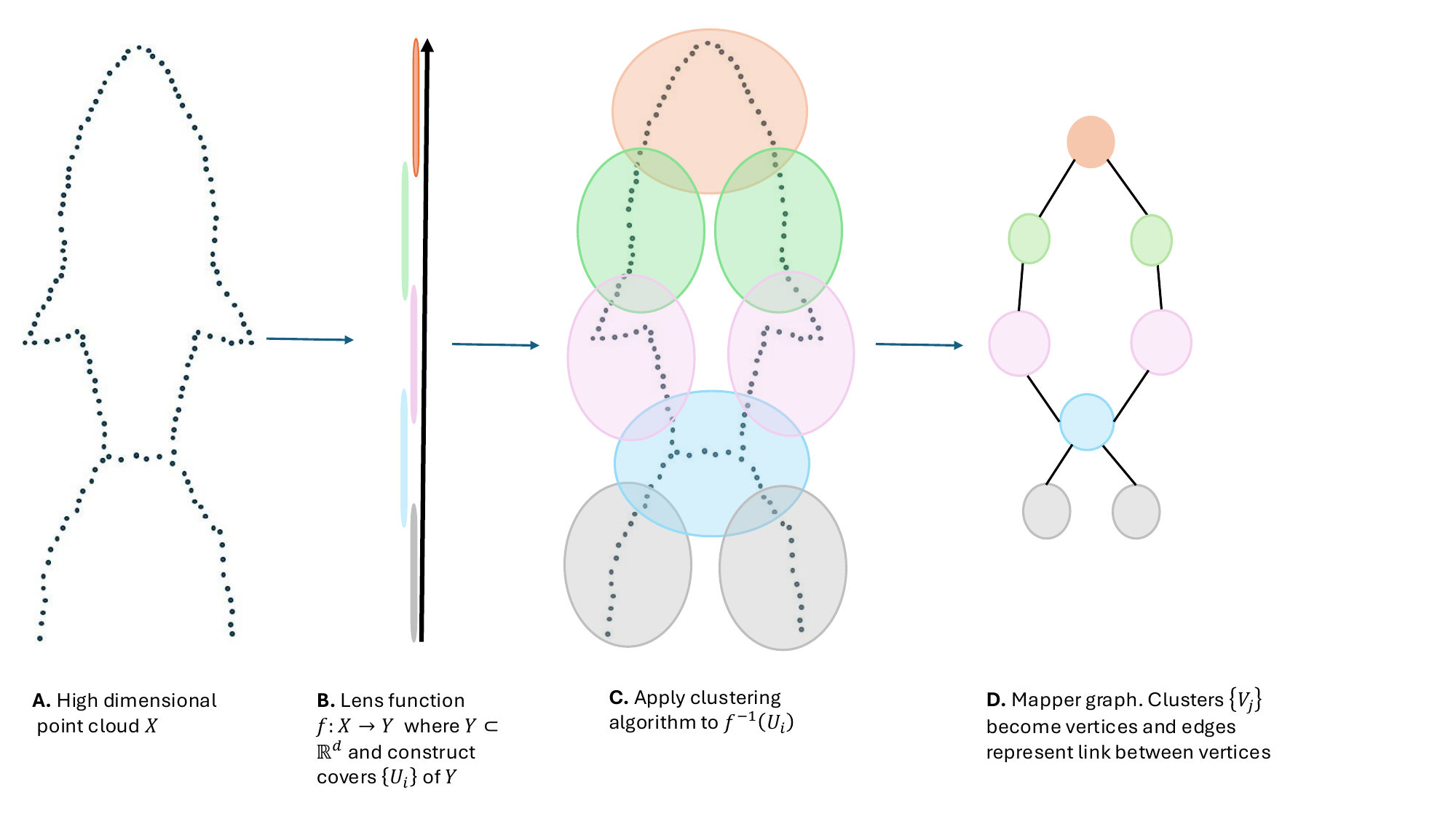}
	\caption{
Overview of the Mapper construction used to obtain topological
representations of datasets. Starting from a high-dimensional dataset
$X$ (A), a lens function maps the data to a lower-dimensional space where
an overlapping cover is constructed (B). Each pullback set is clustered,
and clusters become vertices of the Mapper graph, with edges indicating
overlapping clusters (C). The resulting graph summarizes the multiscale
structural organization of the dataset (D).
}
	\label{fig:mapper}
\end{figure}

The Mapper construction naturally associates each dataset with a
collection of interacting structural components, including refined
pullback regions, overlap relationships, nerve connectivity, and local
attributes. These components define a rich representation whose
mathematical properties can be studied independently of any particular
learning algorithm. The next section formalizes this representation and
develops the mathematical framework used throughout the remainder of
the paper.



\section{Predictive Learning on Mapper Induced Decompositions}
\label{sec:framework}


We investigate whether predictive information is contained not only in the final Mapper graph, but also in these additional components of the Mapper construction. To this end, we treat the entire Mapper induced decomposition as the representation used for learning. This viewpoint allows us to study how predictive information is distributed across different structural components of the Mapper representation, including local statistical summaries, overlap structure, and relational connectivity. 

\subsection{Mapper induced representation spaces}
Let
$
\theta=(f,\mathcal U,\mathcal C)
$
denote the Mapper construction parameters, consisting of a lens
function $f$, a cover of the lens image $\mathcal U$, and a clustering procedure $\mathcal C$.
Associated with a dataset $X$, we define the Mapper induced
representation
\[
\mathcal R_\theta(X)
=
\left(
\mathcal U,
f^{-1}(\mathcal U),
\overline{\mathcal U},
\mathcal N(\overline{\mathcal U}),
\mathrm{Attr}
\right),
\]
where the first four components are described in Section~\ref{sec:prelim}, and $\mathrm{Attr}$ denotes a
collection of attributes attached to elements of the decomposition.
Depending on the application, these attributes may include local
statistical summaries, geometric descriptors, cluster level features,
or other quantities associated with the refined pullback regions and
their interactions.
The usual Mapper graph
\[
M_\theta(X) = 
\operatorname{sk}_1
\big(
\mathcal N(\overline{\mathcal U})
\big)
\]
is therefore only one component of the full representation
$\mathcal R_{\theta}(X)$.
The distinction between
$\mathcal R_{\theta}(X)$
and
$M_\theta(X)$
is fundamental. While the graph records overlap connectivity among
refined clusters, the full representation also retains information
about cover organization, pullback structure, local decomposition
geometry, and associated attributes.

\subsection{Sources of predictive information} \label{sec:sources}

The Mapper induced representation contains several distinct sources of
information that may contribute to prediction.
\begin{itemize}

        \item \emph{Local information.} Each refined pullback region
$
C_i\in\overline{\mathcal U}
$
contains a subset of the original data. Local statistical summaries
computed on these regions may include means, variances, covariance
eigenvalues, signal power, density estimates, cluster sizes, or other
domain specific features.
These summaries capture information about local behavior within the
dataset independently of how the regions are connected.

\item \emph{Structural information.}
A Mapper representation depends on the organization of the underlying decomposition,
including the cover, overlap between neighboring regions, clustering
granularity, and other construction parameters. These structural
properties determine how local regions are formed and related,
influencing the complexity and geometry of the resulting
representation.

\item \emph{Relational information.}
The nerve
$
\mathcal N(\overline{\mathcal U})
$
encodes connectivity among refined pullback regions. This relational
structure captures branching behavior, connectivity patterns, cycles,
and global organization of the decomposition.
Relational information cannot generally be recovered from local
statistics alone.


\end{itemize}

\subsection{Predictive learning on decomposition spaces}

The Mapper induced representation map
\[
\mathcal R_\theta:\mathcal X\to\mathcal R
\]
assigns each input dataset to a structured object in the representation
space $\mathcal R$. Prediction is then performed by applying a learning
map
$
g_\phi:\mathcal R\to\mathcal Y
$
to this representation. Thus the induced predictor on the original data
domain is
$
h_{\theta,\phi} =
g_\phi\circ \mathcal R_\theta,
$
or, equivalently,
\[
h_{\theta,\phi}(X) =
g_\phi\bigl(\mathcal R_\theta(X)\bigr).
\]

This formulation separates the construction of the representation from
the learning rule applied to it. The parameter $\theta$ controls how
the Mapper induced decomposition is formed, while $\phi$ controls how
the resulting structured representation is used for prediction.
Consequently, predictive behavior depends not only on the downstream
model, but also on the decomposition produced by the Mapper
construction.

Different learning rules may use different parts of
$\mathcal R_\theta(X)$. A model using only attributes of refined
pullback regions tests the predictive value of local statistical
summaries. A graph based model using the nerve together with node
attributes tests the role of relational connectivity among local
regions. A model that also incorporates overlap strengths, cover
locations, or multiscale information uses a richer portion of the full
Mapper induced representation.

Thus the framework is not tied to a single predictor architecture. In
our experiments, graph neural networks provide one way to learn from
the attributed nerve structure, while ablated models and perturbed
representations are used to isolate the contribution of different
decomposition components. The purpose is therefore not only to train a
classifier on a Mapper graph, but to study how prediction changes as
different levels of the Mapper induced decomposition become available
to the learning rule.

\subsection{Component wise predictive analysis}

A central objective of the proposed framework is to understand how
different structural components of a Mapper representation contribute
to predictive performance. Rather than evaluating only the accuracy of
the complete representation, we compare it with systematically
modified representations in which selected structural information is
either removed or perturbed. Changes in predictive performance then
provide evidence for the predictive role of the modified component.

In this work, we consider two controlled modifications of the original
Mapper representation. The first replaces the attributed Mapper graph by its disconnected version, preserving all node level attributes while removing edge connectivity. This yields the \emph{No-edge} representation used throughout the experiments. Comparing this
representation with the original Mapper representation isolates the
predictive contribution of relational connectivity. The second
preserves the node attributes but perturbs the graph connectivity,
yielding the \emph{Perturbed} representation. This comparison assesses
whether predictive performance depends on the specific organization of
the Mapper graph rather than merely on the presence of graph
connectivity.

These representation modifications form the basis of the experimental
studies in Section~\ref{sec:exp}. Together with experiments that vary
the Mapper construction parameters, they provide a systematic framework
for analyzing the predictive contribution of structural components,
the robustness of graph connectivity, and the influence of Mapper
construction on learning behavior.















\section{Structural Geometry of Mapper Induced Representations}
\label{sec:geometry}

Section~\ref{sec:framework} introduced Mapper induced representations
as structured objects for learning. Once datasets are mapped into the
representation space
$\mathcal R_\theta(X),$ where $ X\in\mathcal X,$
a fundamental mathematical question is how these representations should
be compared and organized. Such questions naturally arise in tasks such
as similarity search, clustering, retrieval, and kernel based learning,
where comparisons are performed directly on the representation space
rather than on the original datasets.
 Mapper representations consist of
multiple interacting structural components and their geometry cannot be described solely by
vector space distances or inner products. Instead, it requires notions
of equivalence, structural dissimilarity, similarity, and multiscale
organization that are intrinsic to the representation itself.

The purpose of this section is to develop this intrinsic geometry of
Mapper representation spaces. We first remove implementation dependent
labeling through an equivalence relation and quotient construction.
Next, we introduce structural dissimilarities and kernels for comparing
representations. Finally, we develop a refinement hierarchy that
captures the multiscale organization induced by different Mapper
parameter choices. These constructions provide a mathematical framework for comparing and characterizing Mapper representations through their structural organization.

\subsection{Equivalence of Mapper representations} \label{sec:equivalence_mapper}

A fundamental requirement for learning on Mapper induced
representations is invariance under relabeling of refined pullback
regions. The ordering of clusters produced by a clustering algorithm is
an implementation artifact and carries no intrinsic information about
the underlying dataset. Let
\[
\mathcal R_\theta(X)
=
\left(
\mathcal U,
f^{-1}(\mathcal U),
\overline{\mathcal U},
\mathcal N(\overline{\mathcal U}),
A
\right)
\]
and
\[
\mathcal R_{\theta'}(Y)
=
\left(
\mathcal U',
f'^{-1}(\mathcal U'),
\overline{\mathcal U}',
\mathcal N(\overline{\mathcal U}'),
A'
\right)
\]
be two Mapper-induced representations.

\begin{definition}[Representation equivalence]
We say that
$\mathcal R_\theta(X)
\sim
\mathcal R_{\theta'}(Y)$
if there exists a bijection
$\varphi:
\overline{\mathcal U}
\rightarrow
\overline{\mathcal U}'$
such that intersection relations among refined pullback regions,
attributes attached to corresponding regions, and associated overlap
and nerve connectivity relations are preserved.
\end{definition}

The relation \(\sim\) is an equivalence relation on the collection of
labeled Mapper induced representations. We therefore define the Mapper
representation space as the quotient
$
\mathcal R
=
\widetilde{\mathcal R}/\sim,
$
where $\widetilde{\mathcal R}$ denotes the collection of labeled
Mapper induced representations. For each labeled representation
$
R\in\widetilde{\mathcal R},
$
its equivalence class is
$
[R]
=
\{R'\in\widetilde{\mathcal R}:R'\sim R\}.
$. 
The quotient space
$
\mathcal R
$
is the collection of all such equivalence classes.
This quotient construction removes dependence on arbitrary cluster
labels and ensures that subsequent distances, kernels, and complexity
functionals depend on the intrinsic decomposition structure rather than
on implementation specific orderings of refined regions.

The quotient space
$\mathcal R$
is the natural domain for learning and comparison, because two labeled
representations in the same equivalence class differ only by
implementation dependent indexing of refined pullback regions. However,
most quantities are initially computed from a labeled representation:
for example, a complexity score, a graph statistic, an attribute
summary, or a kernel feature map may be evaluated after the refined
regions have been assigned labels. To use such quantities on the
quotient space, we must verify that their values do not change when a
different representative of the same equivalence class is chosen. The
following proposition records this well-definedness condition.
 
\begin{proposition}[Well-defined invariant functionals] \label{prop:well_defined_functional}
Let
$
F:\widetilde{\mathcal R}\to\mathbb R
$
be a functional satisfying
$
F(\mathcal R_1)=F(\mathcal R_2)
\quad
\text{whenever }
\mathcal R_1\sim\mathcal R_2.
$
Then $F$ induces a well-defined functional
\[
\bar F:\mathcal R\to\mathbb R
\]
on the quotient representation space by
$
\bar F([\mathcal R])=F(\mathcal R).
$
\end{proposition}

\begin{proof}
An element of the quotient space $\mathcal R=\widetilde{\mathcal R}/\sim$
is an equivalence class
$[\mathcal R]$. To define $\bar F$ on this class, choose any representative
$\mathcal R\in[\mathcal R]$ and set
$
\bar F([\mathcal R])=F(\mathcal R).$
We need to show that this value is independent of the representative.
Suppose $\mathcal R_1,\mathcal R_2\in[\mathcal R]$. Then
$
\mathcal R_1\sim\mathcal R_2.
$
By the invariance assumption on $F$,
\[
F(\mathcal R_1)=F(\mathcal R_2).
\]
Therefore choosing $\mathcal R_1$ or $\mathcal R_2$ gives the same
value of $\bar F([\mathcal R])$. Hence $\bar F$ is well-defined on
equivalence classes.
\end{proof}

This proposition will be used implicitly in the definitions below:
whenever we define a complexity score, distance, or kernel for
Mapper induced representations, it must either be defined directly on
equivalence classes or arise from a relabeling invariant functional on
labeled representations.

\subsection{Structural distances between Mapper representations}
\label{sec:distance_mapper}

We introduce a family of structural distance functionals
defined on the representation space. Rather than imposing a single
canonical distance, we decompose dissimilarity into multiple
interpretable components corresponding to the major sources of
information identified in Section~ \ref{sec:sources}.
 Table~\ref{tab:distance-components}
summarizes these component wise dissimilarities.

\begin{table}[h!]
\centering
\small
\begin{tabular}{p{2.3cm} p{1.8cm} p{8cm}}
\hline
\textbf{Component} & \textbf{Notation} & \textbf{Interpretation} \\
\hline

Local attribute
&
$
d_{\mathrm{loc}}(R_1,R_2)
$
&
Measures differences between local attributes of refined pullback
regions, such as cluster summaries, statistical moments, geometric
descriptors, or domain specific features.
\\[1ex]

Overlap
&
$
d_{\mathrm{ov}}(R_1,R_2)
$
&
Measures differences in overlap organization among refined regions,
capturing the strength and structure of interactions between
neighboring regions.
\\[1ex]

Relational
&
$
d_{\mathrm{rel}}(R_1,R_2)
$
&
Measures differences in the nerve structure, including connectivity,
branching patterns, and cycle organization.
\\[1ex]


\hline
\end{tabular}
\caption{Component wise structural dissimilarities between Mapper induced representations.}
\label{tab:distance-components}
\end{table}

Each component captures a different aspect of structural variation.
Depending on the application, one or more components may dominate the
representation geometry. For example, in some tasks predictive
differences may be driven primarily by local statistical attributes,
whereas in others relational or multiscale structure may play a more
significant role.
These component wise dissimilarities may be combined into a single
structural distance functional.

\begin{definition}[Structural distance functional] \label{def:structural_dist}
Let
$
\alpha,\beta,\gamma\ge 0
$
be weighting coefficients. The structural distance between two Mapper
representations is defined by $D:\mathcal R\times\mathcal R\to\mathbb R_{\ge 0}$ with
\[
D(R_1,R_2)
=
\alpha d_{\mathrm{loc}}(R_1,R_2)
+
\beta d_{\mathrm{ov}}(R_1,R_2)
+
\gamma d_{\mathrm{rel}}(R_1,R_2)
\]
\end{definition}

The weighting coefficients determine the relative importance assigned
to different structural components. 
To ensure that structural comparisons are intrinsic to the
representation and independent of arbitrary cluster labels, the
distance functional must be invariant under representation
equivalence.

The structural distance framework introduced here equips the
representation space with an intrinsic notion of geometry.
Together with the kernel framework introduced below, it provides a
general mathematical foundation for comparing Mapper
representations independently of any particular predictive model.
Together with the kernel based similarity framework introduced in the
next subsection, it offers complementary tools for comparing
Mapper induced representations.

\subsection{Kernels on representation spaces} \label{sec:kernel_mapper}

While structural distances quantify global dissimilarity between
Mapper induced representations, kernels provide a complementary notion
of similarity based on shared structural characteristics. Rather than
measuring how far two representations are from one another, a kernel
measures how strongly they agree with respect to selected structural
features. This perspective is particularly useful in similarity based
learning methods such as support vector machines, kernel regression,
and nearest neighbor classification.

Kernel methods provide a flexible framework for learning on structured
objects through similarity functions defined implicitly by feature maps
\cite{scholkopf2002kernels}. In the present setting, this allows
learning directly on Mapper induced representations without requiring
explicit vectorization of the full decomposition. Similar ideas have
been used in graph learning, where kernels are constructed from
invariant structural descriptors rather than Euclidean embeddings
\cite{shervashidze2011weisfeiler}. 


\begin{definition}[Structural feature map]
Let $R\in\mathcal R$
be a Mapper induced representation. A structural feature map is a
function $\Phi:\mathcal R\rightarrow\mathbb R^m$ 
whose coordinates encode structural features derived from the
decomposition.
\end{definition}
The structural feature vector $\Phi(R)$
may contain quantities derived from different levels of the Mapper
representation, including local, overlap, and relational structure.
Examples include:
\begin{itemize}
    \item distributions of cluster sizes;
    \item node degree statistics of the nerve graph;
    \item counts of leaf nodes and branch nodes;
    \item cycle-related summaries such as the first Betti number;
    \item overlap multiplicity histograms between neighboring regions.
\end{itemize}

Each component of the structural feature map is chosen to be an
invariant functional in the sense of Section~ \ref{sec:equivalence_mapper}. Thus, if $\Phi(R)=\bigl(F_1(R),\dots,F_m(R)\bigr),$ then each
$F_j:\widetilde{\mathcal R}\to\mathbb R$ 
depends only on intrinsic properties of the decomposition and is
independent of arbitrary relabeling of refined pullback regions.

\begin{proposition}[Well-defined structural feature map]
Let
\[
\Phi(R)=\bigl(F_1(R),\dots,F_m(R)\bigr)
\]
be a structural feature map whose components are invariant functionals.
Then \(\Phi\) is well-defined on the quotient representation space
\(\mathcal R=\widetilde{\mathcal R}/\sim\).
\end{proposition}

\begin{proof}
By Proposition~ \ref{prop:well_defined_functional}, each invariant functional
$F_j:\widetilde{\mathcal R}\to\mathbb R$
induces a well-defined functional on the quotient representation
space. Since every component of
$\Phi$
is well-defined on equivalence classes, the vector-valued map
$\Phi:\mathcal R\to\mathbb R^m$
is also well-defined.
\end{proof}

\begin{definition}[Structural kernel]
Given a structural feature map
$
\Phi:\mathcal R\rightarrow\mathbb R^m,
$
the associated structural kernel is defined by
$
K:\mathcal R\times\mathcal R\rightarrow\mathbb R
$
with
\[
K(R_1,R_2)
=
\langle
\Phi(R_1),\Phi(R_2)
\rangle.
\]
\end{definition}

More generally, one may consider normalized or nonlinear kernels
constructed from structural feature vectors. A common normalized
version is the cosine kernel
\[
K_{\mathrm{cos}}(R_1,R_2)
=
\frac{
\langle \Phi(R_1),\Phi(R_2)\rangle
}{
\|\Phi(R_1)\|\,\|\Phi(R_2)\|
},
\]
which measures structural similarity independently of feature scale.

Unlike structural distance functionals, which emphasize global
differences between decompositions, structural kernels capture shared
structural characteristics and may assign high similarity even to
representations with large structural distance if they exhibit similar
connectivity, overlap organization, or topological summaries. Since the
structural feature map is well-defined on the quotient representation
space, the induced kernel is invariant under representation
equivalence and depends only on intrinsic properties of the
Mapper induced decomposition. Structural kernels therefore provide a
natural mechanism for incorporating Mapper specific structure into
kernel based learning algorithms.

\subsection{Refinement of Mapper representations}

A central feature of Mapper induced representations is their dependence
on the construction parameters
$\theta=(f,\mathcal U,\mathcal C).$
Different choices of cover resolution, overlap, and clustering
granularity may produce substantially different decompositions of the
same dataset. Coarser parameter choices typically yield simpler
representations with fewer refined pullback regions, whereas finer
choices produce more detailed decompositions with increased structural
complexity.

This dependence naturally introduces a multiscale family of
representations for a fixed dataset. Rather than viewing these
representations as unrelated objects, it is useful to organize them
through a refinement relation that captures how one decomposition
resolves another at a finer scale.

\begin{definition}[Refinement relation]
Let
$
R_{\theta_1}(X),\; R_{\theta_2}(X)\in\mathcal R
$
be two Mapper induced representations of the same dataset \(X\).
We say that
$
R_{\theta_2}(X)
$
is a refinement of
$
R_{\theta_1}(X)
$
if every refined pullback region in
$
\overline{\mathcal U}_{\theta_2}
$
is contained in some refined pullback region of
$
\overline{\mathcal U}_{\theta_1}.
$
We write
\[
R_{\theta_1}(X)\preceq R_{\theta_2}(X).
\]
\end{definition}

Intuitively, refinement means that the finer decomposition resolves
regions of the coarser representation into smaller components while
preserving the underlying data organization.
The refinement relation induces a natural hierarchy on the
representation space.
A refinement relation alone provides only pairwise comparisons between
Mapper induced representations. To systematically organize
representations across multiple scales, it is useful to equip the
representation space with an ordering structure. In particular, a
partial ordering provides a mathematically consistent notion of when
one representation is coarser or finer than another, thereby allowing
multiscale Mapper representations to be arranged into hierarchical
chains or more general refinement hierarchies.

\begin{proposition}[Multiscale ordering]
The refinement relation
$
\preceq
$
defines a partial ordering on compatible Mapper induced
representations.
\end{proposition}

\begin{proof}
We verify the three defining properties of a partial order on the
quotient representation space.
 \emph{Reflexivity.}
For any representation
$
R_\theta(X),
$
every refined pullback region
$
C\in\overline{\mathcal U}_\theta
$
satisfies
$
C\subseteq C.
$
Hence
$
R_\theta(X)\preceq R_\theta(X).
$
\emph{Transitivity.}
Suppose
$
R_{\theta_1}(X)\preceq R_{\theta_2}(X)
\quad\text{and}\quad
R_{\theta_2}(X)\preceq R_{\theta_3}(X).
$
For every
$
C_3\in\overline{\mathcal U}_{\theta_3},
$
there exists
$
C_2\in\overline{\mathcal U}_{\theta_2}
$
such that
$
C_3\subseteq C_2,
$
and there exists
$
C_1\in\overline{\mathcal U}_{\theta_1}
$
such that
$
C_2\subseteq C_1.
$
Therefore
$
C_3\subseteq C_1,
$
which implies
$
R_{\theta_1}(X)\preceq R_{\theta_3}(X).
$
\emph{Antisymmetry.}
If
$
R_{\theta_1}(X)\preceq R_{\theta_2}(X)
\quad\text{and}\quad
R_{\theta_2}(X)\preceq R_{\theta_1}(X),
$
then the refined pullback decompositions contain the same regions up to
representation equivalence. Hence the corresponding equivalence classes
coincide in the quotient representation space.
\end{proof}

The refinement relation provides a mathematical ordering of Mapper
representations generated under different construction parameters,
formalizing the notion that finer parameter choices yield more detailed
decompositions of the same dataset. Although the experiments do not
explicitly verify this ordering, they examine families of
parameter-dependent representations using the structural distances
introduced in this section. The resulting distance matrices and
hierarchical clustering provide an empirical view of the relationships
among these representations.

Taken together, the constructions developed in this section equip the
space of Mapper induced representations with an intrinsic mathematical
geometry through representation equivalence, structural distances,
kernel based similarities, and refinement. These concepts provide a
principled framework for comparing and organizing Mapper
representations independently of any particular learning algorithm. The
next section complements this geometric viewpoint by introducing
representation complexity and studying the stability of graph-based
predictors acting on Mapper representations.






\section{Complexity and Stability of Mapper Representations}
\label{sec:complexity_stability}

The predictive utility of a Mapper-induced representation depends not
only on the structural information it encodes, but also on its
complexity and robustness. Highly coarse decompositions may fail to
capture informative local structure, while excessively fine
decompositions may become sensitive to noise, unstable under
perturbations, and prone to overfitting. Understanding this trade-off
is therefore essential for principled learning on Mapper
representations.
This section develops quantitative notions of representation
complexity and stability for Mapper-induced decompositions. 

\subsection{Topological complexity and overlap interaction}
\label{sec:complexity}

A central question in decomposition-aware learning is how to quantify
the complexity of a Mapper induced representation. Unlike classical
learning settings, where complexity is often associated with ambient
dimension or model capacity, complexity in the present framework
emerges from the structure of the Mapper decomposition itself. In
particular, a representation may become more complex either by
increasing the number of refined pullback regions, by developing richer
topological connectivity, or by exhibiting stronger overlap
interactions among neighboring regions.
Quantifying representation complexity is 
essential for understanding the trade-off between structural richness
and predictive robustness.

The proposed complexity functional is motivated by the observation
that structural complexity in Mapper representations arises from three
distinct mechanisms. First, increasing the number of refined pullback
regions increases the resolution of the decomposition and the number of
local objects that must be represented. Second, richer connectivity of
the induced nerve graph produces increasingly complex global
organization, reflected by the presence of cycles and branching
patterns. Third, the overlap among neighboring regions determines the
extent to which local structures interact. These three mechanisms
capture complementary aspects of Mapper representations---local
resolution, global topology, and interaction structure---and together
motivate the complexity measure introduced below.
Guided by these three structural mechanisms, we define the topological
complexity of a Mapper representation $R_\theta$ by
\[
TC(R_\theta(X))= 
a|\overline{\mathcal U}|
+
b\beta_1
+
c\rho_{\mathrm{overlap}},
\]
where $a,b,c>0$ are weighting coefficients, and
\begin{itemize}
\item
$|\overline{\mathcal U}|$
denotes the number of refined pullback regions and measures local
decomposition complexity;

\item $\beta_1$
denotes the first Betti number of the induced Mapper graph and captures
global topological complexity through cycle structure;

\item
$
\rho_{\mathrm{overlap}}
$
measures the aggregate overlap interaction among neighboring refined
regions.
\end{itemize}


To quantify overlap interaction, let
$C_1,\ldots,C_m\in\overline{\mathcal U}$
denote the refined pullback regions. In this work we measure overlap by
the normalized overlap ratio
\[
\rho_{\mathrm{overlap}}
=
\frac{
\sum_{i<j}|C_i\cap C_j|
}{
\sum_{i=1}^{m}|C_i|
},
\]
which represents the total amount of shared membership relative to the
total size of the refined pullback regions. Larger values indicate
stronger interaction among neighboring regions through overlap. 
Other overlap measures, such as
the raw overlap count
$
|C_i\cap C_j|,
$
the Jaccard overlap
$
\frac{|C_i\cap C_j|}{|C_i\cup C_j|},
$
or the minimum-normalized overlap
$
\frac{|C_i\cap C_j|}{\min(|C_i|,|C_j|)}.
$


\subsection{Predictive stability of Mapper graphs}
\label{sec:stability}

The structural geometry developed in Section~\ref{sec:geometry}
concerns intrinsic comparison of complete Mapper induced
representations. Here we study a complementary problem: the stability
of graph based predictors acting on the attributed Mapper graph derived
from these representations. Whereas classical stability results for
Mapper investigate how perturbations of the input data or Mapper
construction affect the resulting topological summary
\cite{carriere2018mapper,carriere2018statistical}, our objective is to
understand how perturbations of the attributed Mapper graph propagate
through a downstream learning model and influence its predictions.

For graph based predictors, we represent the Mapper graph
$M_\theta(X)$ by its weighted adjacency matrix
$A_\theta$
together with the corresponding node-feature matrix
$X_\theta$,
whose rows contain the attributes associated with the refined pullback
regions. Thus,
\[
M_\theta(X)\equiv(A_\theta,X_\theta),
\]
and the following analysis concerns predictors acting on this
attributed graph representation. Extending the stability analysis to
learning architectures that operate directly on the complete
representation
$R_\theta(X)$,
including the cover, pullback organization, and overlap structure,
remains an important direction for future work.


     Let
$M_\theta=(A_\theta,X_\theta)$ and
$\widetilde M_\theta =(\widetilde A_\theta,\widetilde X_\theta)
$
denote the original and perturbed attributed-graph projections. 
    We measure the perturbation magnitude by
    $
    \|A_\theta-\widetilde A_\theta\|_2
    \quad\text{and}\quad
    \|X_\theta-\widetilde X_\theta \|_F
    $. Let $H^{(k)}\in\mathbb R^{m\times d_k}$
denote the node embedding matrix after the $k$-th message-passing
layer, where each row corresponds to the feature representation of a
refined pullback region. 
Let $H^{(0)}=X_\theta ,$ $\widetilde H^{(0)}=\widetilde X_\theta ,$
and define the two message-passing sequences by
\[
H^{(k+1)}
=
\sigma\!\left(A_\theta H^{(k)}W_k\right),
\qquad
\widetilde H^{(k+1)}
=
\sigma\!\left(\widetilde A_\theta \widetilde H^{(k)}W_k\right),
\]
for \(k=0,\ldots,K-1\). The same trained weight matrices \(W_k\)
are used in both evaluations; thus, the analysis isolates sensitivity
to the input representation rather than variation caused by
retraining the model.

Let
$
g_\phi(M_\theta)=g_\phi(A_\theta,X_\theta)
$
denote a graph based predictor acting on the attributed Mapper graph.
Following the standard message passing framework, node embeddings are
first computed through successive graph convolution layers. A graph level
representation is then obtained by applying a pooling operator
$P:\mathbb{R}^{m\times d_K}\rightarrow\mathbb{R}^{d},$
which aggregates the final node embeddings into a single graph
embedding. Finally, an output map $q:\mathbb{R}^{d}\rightarrow\mathcal Y$
produces the prediction. Thus,
\[
g_\phi(A_\theta,X_\theta)
=
q\!\left(P(H^{(K)})\right).
\]
To establish the perturbation bound, we impose the following standard assumptions on the graph predictor. These assumptions are commonly used in the stability analysis of neural networks and graph neural networks. The norms appearing below are the natural norms associated with the domains and codomains of the corresponding mappings.
\[
\|\sigma(U)-\sigma(V)\|_F
\le L_\sigma\|U-V\|_F,
\]
\[
\|P(U)-P(V)\|_2
\le L_P\|U-V\|_F,
\]
and
\[
d_{\mathcal Y}(q(u),q(v))
\le L_q\|u-v\|_2.
\]
Suppose further that
\[
\|A_\theta \|_2,\|\widetilde A_\theta \|_2\le M_A,
\qquad
\|W_k\|_2\le M_k,
\]
and
\[
\|H^{(k)}\|_F,\|\widetilde H^{(k)}\|_F\le B_k
\]
for every layer \(k\).



\begin{proposition}[Layerwise perturbation bound]\label{prop:layerwise_bound}
Let
$
\Delta_k
=
\|H^{(k)}-\widetilde H^{(k)}\|_F.
$
Then, for \(k=0,\ldots,K-1\),
\[
\Delta_{k+1}
\le
L_\sigma M_k
\left(
M_A\Delta_k
+
B_k\|A_\theta -\widetilde A_\theta \|_2
\right).
\]
\end{proposition}

\begin{proof}
Using the definition of the two message-passing layers,
$H^{(k+1)}$ and $\widetilde H^{(k+1)}$, and
by Lipschitz continuity of \(\sigma\),
\[
\Delta_{k+1}
\le
L_\sigma
\|A_\theta H^{(k)}W_k-\widetilde A_\theta \widetilde H^{(k)}W_k\|_F.
\]
Using submultiplicativity,
\[
\Delta_{k+1}
\le
L_\sigma \|W_k\|_2
\|A_\theta H^{(k)}-\widetilde A_\theta \widetilde H^{(k)}\|_F.
\]
Now write
\[
A_\theta H^{(k)}-\widetilde A_\theta \widetilde H^{(k)}
=
A_\theta (H^{(k)}-\widetilde H^{(k)})
+
(A_\theta -\widetilde A_\theta )\widetilde H^{(k)}.
\]
Therefore,
\[
\|A_\theta H^{(k)}-\widetilde A_\theta \widetilde H^{(k)}\|_F
\le
\|A_\theta \|_2\Delta_k
+
\|A_\theta -\widetilde A_\theta \|_2\|\widetilde H^{(k)}\|_F.
\]
Using
\[
\|A_\theta \|_2\le M_A,
\qquad
\|W_k\|_2\le M_k,
\qquad
\|\widetilde H^{(k)}\|_F\le B_k,
\]
we obtain
\[
\Delta_{k+1}
\le
L_\sigma M_k
\left(
M_A\Delta_k
+
B_k\|A_\theta -\widetilde A_\theta \|_2
\right).
\]
\end{proof}

Proposition~\ref{prop:layerwise_bound} provides a recursive estimate
for the propagation of perturbations through successive
message passing layers. Although sufficient for establishing prediction
stability, the recursive form obscures the cumulative influence of
perturbations as the network depth increases. By repeatedly applying
the layerwise estimate, the recursion can be unrolled to obtain an
explicit bound on the final embedding perturbation. This expression
separates the effects of the initial feature perturbation from the
accumulated adjacency perturbations and makes their dependence on the
network parameters transparent.

\begin{theorem}[Prediction stability under attributed-graph perturbations]
\label{thm:prediction_stability}
Let $M_\theta$ and $\widetilde M_\theta$
be two attributed graph projections evaluated by the same \(K\)-layer
message passing predictor. Suppose
\[
\|A_\theta -\widetilde A_\theta \|_2\le\varepsilon_A
\qquad\text{and}\qquad
\|X_\theta -\widetilde X_\theta \|_F\le\varepsilon_X.
\]
Under the stated Lipschitz and boundedness assumptions,
\[
d_{\mathcal Y}
\bigl(
g_\phi(A_\theta ,X_\theta ),
g_\phi(\widetilde A_\theta ,\widetilde X_\theta )
\bigr)
\le
L_qL_P\Delta_K,
\]
where
$
\Delta_0=\varepsilon_X.
$
\end{theorem}

\begin{proof}
The recursive bound on \(\Delta_K\) follows from 
proposition \ref{prop:layerwise_bound} and the assumptions
\[
\|A_\theta -\widetilde A_\theta \|_2\le \varepsilon_A,
\qquad
\|X_\theta -\widetilde X_\theta \|_F\le \varepsilon_X.
\]
We use the Lipschitz continuity of \(q\) and \(P\) to obtain
\[
d_{\mathcal Y}
\bigl(
g_\phi(A_\theta ,X_\theta ),g_\phi(\widetilde A_\theta ,\widetilde X_\theta )
\bigr)
\le
L_q
\|P(H^{(K)})-P(\widetilde H^{(K)})\|_2
\le
L_qL_P
\|H^{(K)}-\widetilde H^{(K)}\|_F.
\]
Thus,
\[
d_{\mathcal Y}
\bigl(
g_\phi(A_\theta ,X_\theta ),g_\phi(\widetilde A_\theta ,\widetilde X_\theta )
\bigr)
\le
L_qL_P\Delta_K.
\]
\end{proof}

Theorem~\ref{thm:prediction_stability} expresses prediction stability in
terms of the perturbation accumulated across the message passing
layers. For interpretation and subsequent analysis, it is useful to
eliminate this recursion and obtain an explicit dependence on the
perturbations of the adjacency matrix and node attributes. The
following corollary provides this closed-form bound.

\begin{corollary}[Explicit prediction perturbation bound]
Under the assumptions of Theorem~\ref{thm:prediction_stability}, let $a_k=L_\sigma M_kM_A$ and $b_k=L_\sigma M_kB_k.$
Then
\[
d_{\mathcal Y}
\bigl(
g_\phi(A_\theta ,X_\theta ),
g_\phi(\widetilde A_\theta ,\widetilde X_\theta )
\bigr)
\le
L_qL_P
\left[
\left(\prod_{j=0}^{K-1}a_j\right)\varepsilon_X
+
\varepsilon_A
\sum_{i=0}^{K-1}
b_i
\prod_{j=i+1}^{K-1}a_j
\right].
\]
\end{corollary}

\begin{proof}
From Theorem~\ref{thm:prediction_stability}, $\Delta_{k+1}
\le
a_k\Delta_k+b_k\varepsilon_A$, 
where $a_k=L_\sigma M_kM_A$, $b_k=L_\sigma M_kB_k,$ 
and $\Delta_0=\varepsilon_X$.
Applying the recursion successively gives
\[
\Delta_1
\le
a_0\varepsilon_X+b_0\varepsilon_A,
\]
and
\[
\Delta_2
\le
a_1\Delta_1+b_1\varepsilon_A
\le
a_1a_0\varepsilon_X
+
\left(a_1b_0+b_1\right)\varepsilon_A.
\]

Continuing recursively yields
\[
\Delta_K
\le
\left(\prod_{j=0}^{K-1}a_j\right)\varepsilon_X
+
\varepsilon_A
\sum_{i=0}^{K-1}
b_i
\prod_{j=i+1}^{K-1}a_j,
\]
where an empty product is interpreted as \(1\).

Finally, Theorem~\ref{thm:prediction_stability} gives
\[
d_{\mathcal Y}
\bigl(
g_\phi(A_\theta ,X_\theta ),
g_\phi(\widetilde A_\theta ,\widetilde X_\theta )
\bigr)
\le
L_qL_P\Delta_K.
\]
Substituting the bound above proves the result.
\end{proof}

The preceding results characterize how perturbations of the attributed
Mapper graph propagate through a graph based predictor. Perturbations
of the node feature matrix affect the local information associated with
refined pullback regions, whereas perturbations of the adjacency matrix
modify the relational structure used during message passing. The
explicit perturbation bound quantifies how these sources of variation
accumulate across successive network layers and influence the final
prediction.

This analysis concerns the stability of graph based predictors acting
on attributed Mapper graphs rather than the stability of the Mapper
construction itself. The perturbation experiments in
Section~\ref{sec:exp} provide an empirical counterpart by introducing
controlled modifications of the Mapper graph connectivity and
examining how these structural perturbations affect predictive
performance across different datasets.










\begin{figure}[t]
\centering
\includegraphics[width=.95\textwidth]{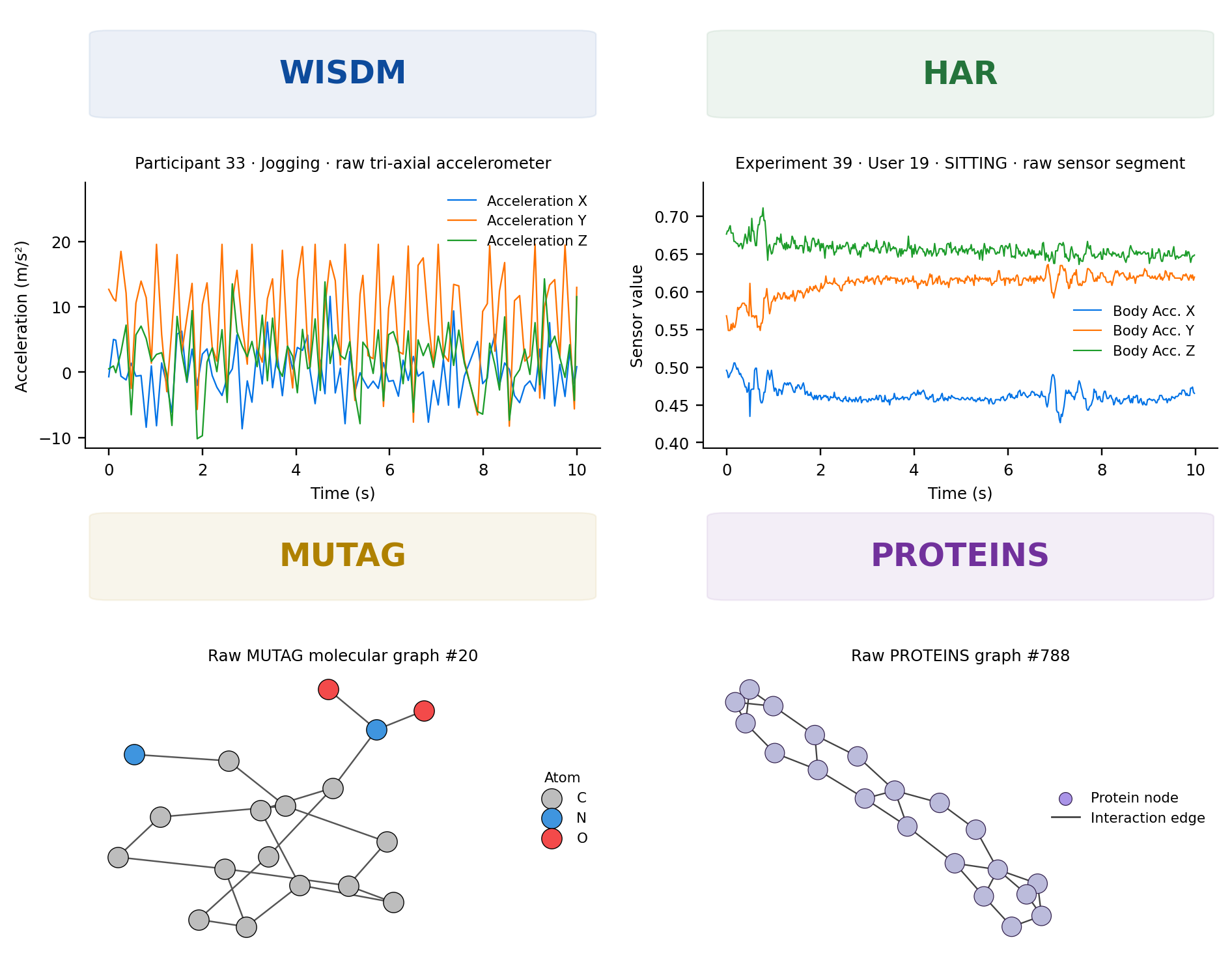}
\caption{
Representative samples from the four benchmark datasets used in the
experimental evaluation.
}
\label{fig:datasets}
\end{figure}

\section{Experiments}
\label{sec:exp}

The previous sections developed three complementary aspects of learning
on Mapper induced representations: representation richness, structural
geometry, and representation complexity. The experiments in this
section investigate these ideas through empirical evaluation on four
benchmark datasets. Throughout the experiments, we study two
complementary dimensions of Mapper representations. The first concerns
the amount of structural information retained by the representation,
ranging from global summaries to the complete Mapper graph. The second
concerns the Mapper construction itself, where the number of cover
intervals and overlap percentage are varied to generate families of
related representations.

The experimental study is organized around the following research
questions:

\begin{enumerate}
    \item \textbf{Representation richness.}
How does predictive performance change as progressively richer
structural information from the Mapper representation is retained?

    \item \textbf{Representation geometry.}
Do the proposed structural distances meaningfully organize Mapper
representations generated by varying the cover resolution and overlap
percentage?

    \item \textbf{Complexity and robustness.}
How do representation complexity and controlled perturbations of the
Mapper graph influence predictive performance?

\end{enumerate}

\subsection{Experimental Design}
\label{subsec:experimental_design}

To evaluate the proposed framework across multiple data modalities, we
consider four benchmark datasets representing both multivariate
time series classification and graph classification problems.
Specifically, WISDM~\cite{kwapisz2011activity} and
HAR~\cite{anguita2013public} are human activity recognition datasets
constructed from wearable sensor measurements, while
MUTAG~\cite{debnath1991structure,morris2020tudataset} and
PROTEINS~\cite{borgwardt2005protein,morris2020tudataset} are widely used
graph classification benchmarks. Representative examples from the four
datasets are shown in Figure~\ref{fig:datasets}. These datasets exhibit substantially
different structural characteristics and therefore provide an
opportunity to evaluate whether the proposed framework captures useful
representation properties across heterogeneous domains rather than a
single application.

\begin{table}[h!]
\centering
\caption{Fixed Mapper construction and learning settings used in the
experiments. 
}
\label{tab:fixed_experimental_settings}
\small
\resizebox{\textwidth}{!}{
\begin{tabular}{lcccc}
\toprule
\textbf{Dataset}
&
\textbf{Lens}
&
\textbf{Distance}
&
\textbf{Clustering method}
\\
\midrule

HAR
&
\textit{ RMS of acc. and gyr. coordinates (2D) }
&
\textit{Euclidean}
&
\textit{HDBSCAN}
\\

WISDM
&
\textit{Acceleration RMS}
&
\textit{Euclidean}
&
\textit{HDBSCAN}
\\

MUTAG
&
\textit{Spectral}
&
\textit{induced-graph connectivity}
&
\textit{Connected component}
\\

PROTEINS
&
\textit{diffusion-map lenses (2D)}
&
\textit{induced-graph connectivity}
&
\textit{Connected component}
\\
\bottomrule
\end{tabular}
}
\end{table}
For every dataset, a Mapper representation is constructed using a fixed
filter function, cover construction, clustering algorithm, edge
weighting, and node feature definition. Unless otherwise stated, all
prediction experiments employ graph neural network
architecture, ensuring that differences in predictive performance are
attributable to the Mapper representations rather than changes in the
learning model. Table~\ref{tab:fixed_experimental_settings} summarizes
the Mapper construction settings that remain fixed for each dataset,
while dataset specific node attribute categories are summarized in
Appendix~\ref{app:node_features}.
For the parameter sensitivity
studies, the number of cover intervals was varied over
$\{5,7,9,11,13,15\}$ and the overlap percentage over
$\{10\%,20\%,30\%,40\%,50\%\}$ for all datasets, while all remaining
Mapper construction settings were held fixed.

\paragraph{Learning protocol.}
For the attributed graph representations, prediction is performed using
a two-layer graph convolutional network (GCN) followed by global mean
pooling and a linear classifier. The Global representation, which does
not retain a graph structure, is instead classified using a support
vector machine (SVM). For WISDM and HAR, cross-validation is performed
at the subject level, so that samples from the same subject do not
appear in both training and test folds. For MUTAG and PROTEINS, we use
stratified ten-fold cross-validation. Model and optimization
hyperparameters are fixed within each dataset group and are reported in
Appendix~\ref{app:learning_details}.

A central objective of this paper is to investigate how different
components of a Mapper representation contribute to predictive
learning. To this end, we consider the richness of representations
illustrated in Figure~\ref{fig:representation_hierarchy}. Beginning with
coarse global summaries and progressively incorporating local node
attributes and relational connectivity, this richness provides a
systematic mechanism for isolating the predictive contribution of
individual structural components. 
As a reference, we additionally report predictive performance on the
original input data without constructing a Mapper representation. Complete numerical results are provided in
Appendix~\ref{app:raw_results}. Since all four datasets are established
benchmark datasets, these reference results provide context for the
underlying prediction task. They are reported separately because the
proposed framework investigates learning within Mapper induced
representations rather than comparisons with conventional feature based
learning pipelines. Accordingly, all subsequent experiments focus on
comparisons among the proposed Mapper representation variants.

Classification performance is evaluated using ten-fold
cross-validation. We report Macro-F1 because of the class imbalance present in several benchmark
datasets. Unless explicitly stated otherwise, reported values represent
averages across all cross-validation folds.
\begin{figure}[h!]
\centering
\includegraphics[width=.95\textwidth]{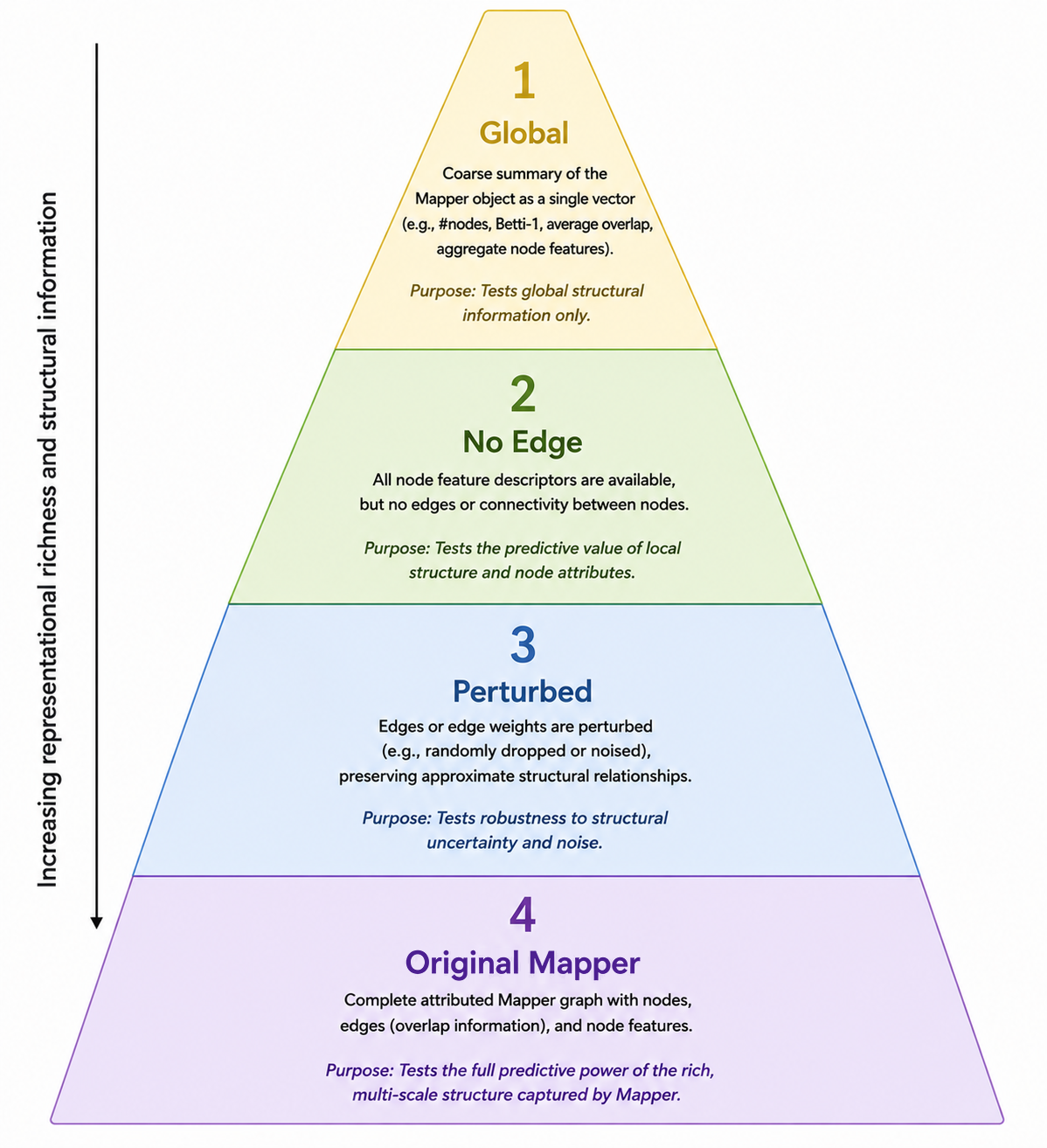}
\caption{Richness of Mapper induced representations considered in the experimental study.}
\label{fig:representation_hierarchy}
\end{figure}
The remainder of this section follows the organization outlined above.








\subsection{Predictive contribution of representation components}
\label{subsec:representation_hierarchy_results}


The first experiment evaluates the predictive contribution of the
different levels of representation richness introduced in
Figure~\ref{fig:representation_hierarchy}. By comparing the four
representation variants, we investigate how progressively retaining
different structural components of the Mapper representation influences
predictive performance.
For each dataset and Mapper parameter setting, all representations are
constructed from the same underlying Mapper object. The input data,
training and test partitions, node attributes, predictive architecture,
and optimization procedure are held fixed wherever the representations
are directly comparable. Consequently, the comparison isolates the
effect of the structural information retained by each representation.
Performance is evaluated using test Macro-F1, with accuracy and balanced
accuracy reported as supplementary measures.

\begin{figure}[t]
    \centering
    \includegraphics[width=0.92\textwidth]{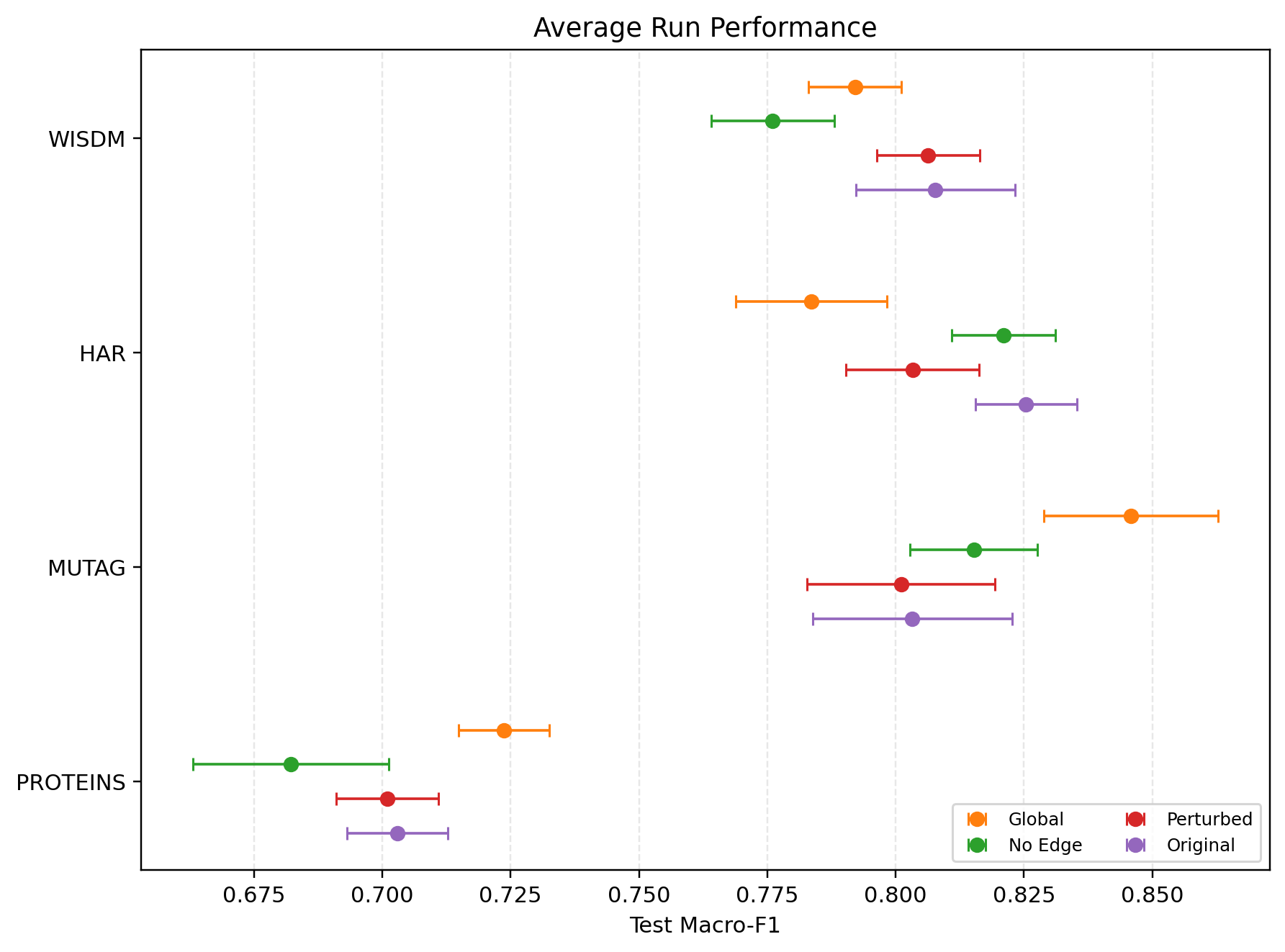}
    \caption{Average predictive performance of the four representation
variants across the evaluated Mapper parameter configurations.
Points indicate mean test Macro-F1 across the 30 configurations, and
error bars indicate one standard deviation across configurations.
Thus, the error bars characterize sensitivity to Mapper construction
parameters rather than variability across repeated training runs.}
    \label{fig:representation_summary}
\end{figure}





Figure~\ref{fig:representation_summary} summarizes the predictive
performance across the four datasets. Parameter specific performance for
individual cover resolutions and overlap percentages is provided in
Appendix~\ref{app:parameter_heatmaps}. In addition to the average performance, we report the maximum test
Macro-F1 attained over the common parameter grid in Table~\ref{tab:best_parameter_settings}. While the average
performance characterizes the overall behavior of each representation,
the maximum value illustrates its best achievable performance within the
range of Mapper constructions considered in this study.

\begin{table}[h!]
\centering
\caption{Maximum observed test Macro-F1 over the evaluated Mapper parameter grid}
\label{tab:best_parameter_settings}
\small
\begin{tabular}{llccc}
\textbf{Dataset}
&
\textbf{Representation}
&
\textbf{Intervals}
&
\textbf{Overlap}
&
\textbf{Test Macro-F1}
\\
\midrule
WISDM
    & Global          & 7  & 0.6 & \(0.807 \pm 0.055\) \\
    & No-edge         & 10 & 0.2 & \(0.796 \pm 0.060\) \\
    & Edge-perturbed  & 16 & 0.6 & \(0.826 \pm 0.066\) \\
    & Original Mapper & 7  & 0.2 & \(0.832 \pm 0.057\) \\
\midrule
HAR
    & Global          & 13 & 0.5 & \(0.804 \pm 0.000\) \\
    & No-edge         & 7  & 0.6 & \(0.852 \pm 0.037\) \\
    & Edge-perturbed  & 19 & 0.2 & \(0.835 \pm 0.049\) \\
    & Original Mapper & 19 & 0.4 & \(0.840 \pm 0.046\) \\
\midrule
MUTAG
    & Global          & 19 & 0.6 & \(0.871 \pm 0.087\) \\
    & No-edge         & 7  & 0.6 & \(0.794 \pm 0.106\) \\
    & Edge-perturbed  & 7  & 0.3 & \(0.856 \pm 0.071\) \\
    & Original Mapper & 13 & 0.5 & \(0.839 \pm 0.067\) \\
\midrule
PROTEINS
    & Global          & 19 & 0.4 & \(0.740 \pm 0.031\) \\
    & No-edge         & 7  & 0.6 & \(0.717 \pm 0.047\) \\
    & Edge-perturbed  & 10 & 0.5 & \(0.717 \pm 0.046\) \\
    & Original Mapper & 7  & 0.5 & \(0.729 \pm 0.044\) \\
\bottomrule
\end{tabular}
\end{table}

For each dataset, Mapper representations were generated over a common
parameter grid as previously mentioned. 
For every construction, the four representation variants
(Global Summary, No-edge, Perturbed, and Original Mapper) were
evaluated using identical training/test partitions, node attributes,
graph neural network architecture, and optimization procedure. Unless
otherwise stated, the reported performance is the average over all 30
Mapper constructions, thereby isolating the effect of the retained
structural information rather than a particular choice of Mapper
parameters. Performance is evaluated using test Macro-F1, with accuracy
and balanced accuracy reported as supplementary measures.

The Edge-perturbed representation generally performs below the
corresponding Original Mapper representation, although the magnitude of the
difference depends on the dataset and parameter setting. This suggests
that the precise nerve connectivity may contribute useful predictive
information, but its effect cannot be characterized solely by the
presence or absence of edges. The detailed parameter level comparisons
in Appendix~\ref{app:parameter_heatmaps} show that the contribution of
connectivity also depends on cover resolution and overlap.

The representation richness reveals that predictive information is
distributed across multiple components of the Mapper construction.
Local attributes account for a substantial portion of performance,
while relational structure provides an additional, but
dataset dependent, contribution. Thus, progressively richer
representations do not guarantee monotone improvement; rather, the
utility of each component depends on the data and the scale at which
the Mapper representation is constructed.
Furthermore, the No-edge, Perturbed, and Original representations are evaluated
under the same GCN architecture, enabling controlled comparisons of
the effect of graph connectivity. The Global representation is a
single vector summary and is therefore evaluated using an SVM; its
comparison with the graph based representations should consequently
be interpreted as a broader representation level baseline rather than
a strict architectural ablation.

\subsection{Empirical organization of the representation space}
\label{subsec:representation_geometry_results}


Section~\ref{sec:geometry} decomposes dissimilarity between Mapper
representations into local, overlap, and relational components. To
examine the empirical behavior of this geometry, we construct Mapper
representations over the parameter grid and compute pairwise structural
distances using Definition \ref{def:structural_dist} on HAR as a representative dataset.
All component descriptors are invariant under relabeling of refined
pullback regions. Before combination, descriptor coordinates are
standardized to prevent quantities with larger numerical scales from
dominating the composite distance. Unless otherwise stated, equal
weights
$
\alpha=\beta=\gamma=\frac{1}{3}
$
are used. Although equal weights are used throughout this work, the empirical descriptor magnitudes are not identical. For the HAR dataset, the mean pairwise descriptor distances are
\[
\overline d_{\rm loc}=0.21,\qquad
\overline d_{\rm ov}=0.31,\qquad
\overline d_{\rm rel}=0.47,
\]indicating that the three descriptors capture complementary aspects of structural variation rather than redundant information.

For a pair of parameter settings \(\theta_s\) and \(\theta_t\), the setting level
distance is obtained by averaging the distance between matched
representations:
\[
\overline D(R_{\theta_s}, R_{\theta_t})
=
\frac{1}{n}
\sum_{i=1}^{n}
D\bigl(R_{\theta_s}(X_i),R_{\theta_t}(X_i)\bigr).
\]
The resulting distance matrix describes the organization of the
parameterized family of Mapper representations independently of the
downstream predictor.

\begin{figure}[t]
    \centering

    \begin{minipage}{0.49\textwidth}
        \centering
        \includegraphics[width=\linewidth]
        {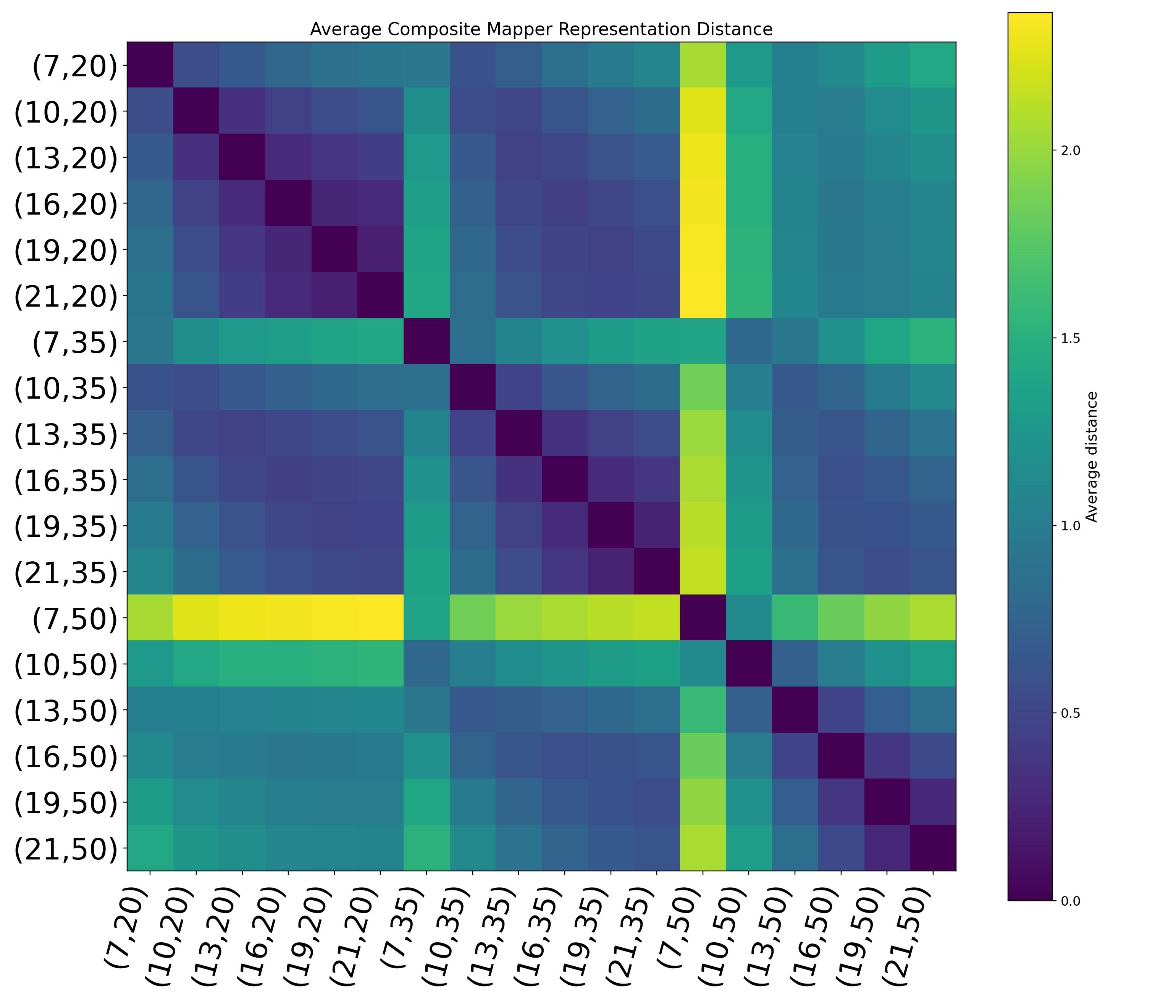}

        \small (a) Pairwise structural distances
    \end{minipage}
    \hfill
    \begin{minipage}{0.49\textwidth}
        \centering
        \includegraphics[width=\linewidth]{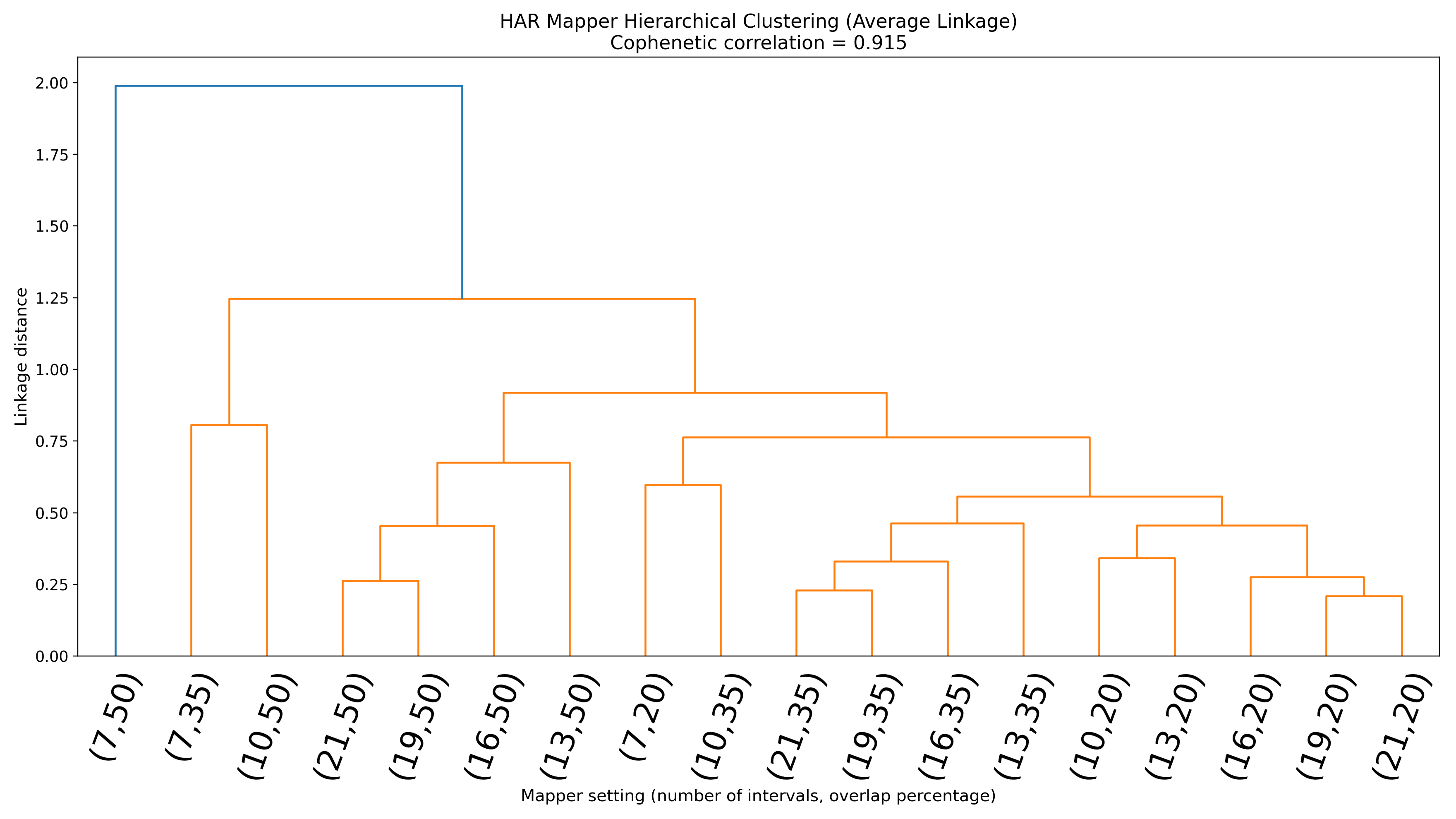}

        \small (b) Hierarchical clustering
    \end{minipage}

    \caption{Structural organization of Mapper representations across
    the parameter grid. The heatmap displays average pairwise composite
    distances, while the dendrogram summarizes the hierarchy induced by
    those distances.}
    \label{fig:representation_geometry}
\end{figure}

\begin{figure}[t]
    \centering
    \includegraphics[width=0.68\textwidth]
    {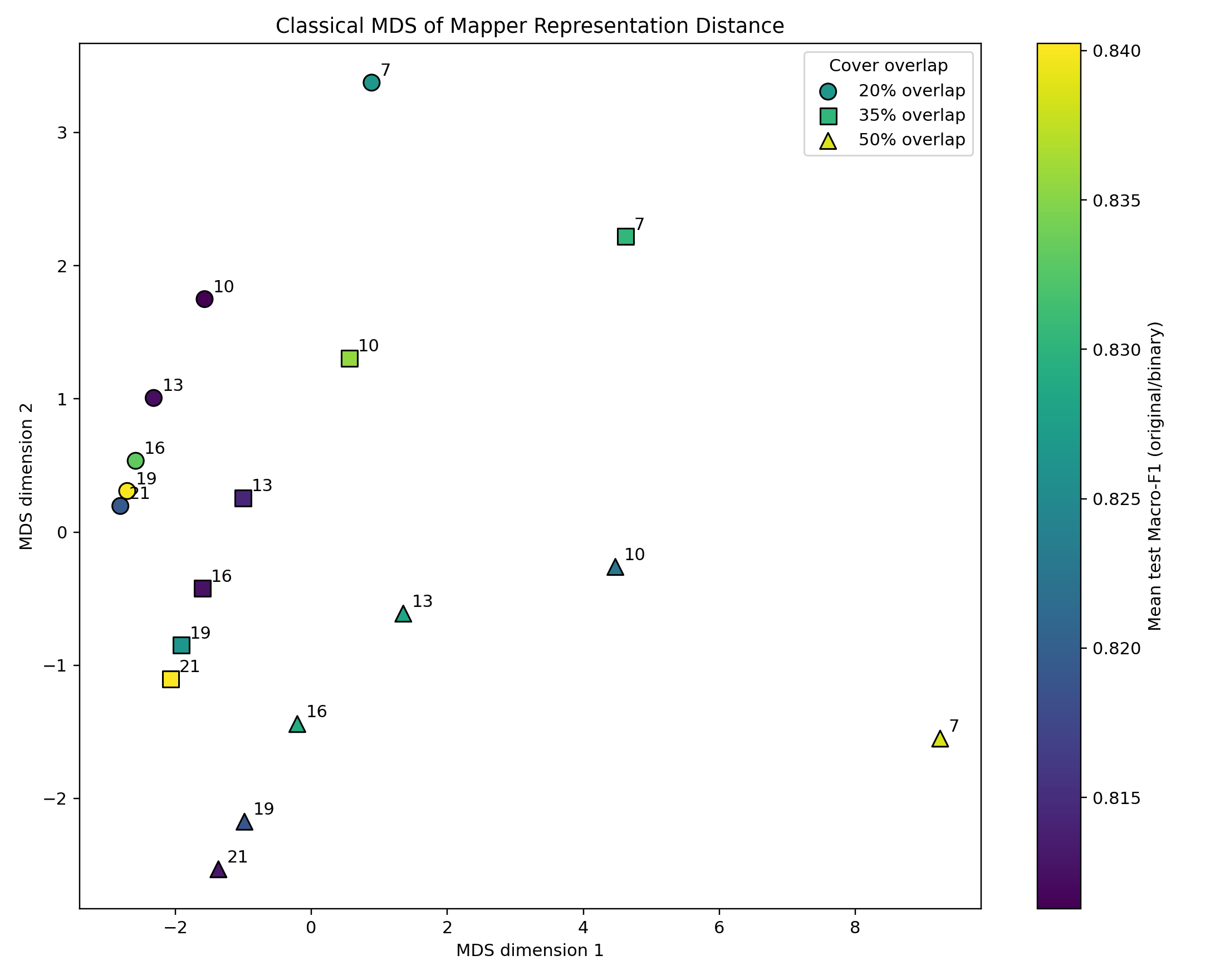}
    \caption{Two-dimensional classical MDS
    visualization of the Mapper representation distance matrix.
    Marker shape identifies cover overlap and labels indicate the number
    of cover intervals.}
    \label{fig:representation_mds}
\end{figure}

Figure~\ref{fig:representation_geometry} summarizes the organization induced by the proposed representation distance. The heatmap displays the pairwise composite distances between Mapper representations generated across the parameter grid, while the dendrogram organizes these representations according to average-linkage hierarchical clustering. The resulting hierarchy exhibits a high cophenetic correlation (0.915), indicating that the clustering provides a faithful summary of the underlying distance matrix.

The analysis reveals several systematic patterns. Representations generated under nearby parameter settings generally have smaller structural distances, whereas coarse constructions using seven cover intervals are consistently separated from higher-resolution representations. Moreover, the organization is not determined by cover resolution alone. Changes in overlap interact with the number of intervals, producing groups that reflect the combined effect of the Mapper construction parameters rather than either parameter individually.
Figure~\ref{fig:representation_mds} provides a complementary geometric visualization using classical multidimensional scaling. Unlike the dendrogram, which emphasizes hierarchical relationships, the multidimensional scaling (MDS) embedding reveals the global geometry of the representation space. Nearby Mapper constructions remain close in the embedding, while structurally distinct parameter settings occupy well-separated regions. Together, these visualizations demonstrate that the proposed distance induces a coherent organization of Mapper representations rather than an arbitrary collection of pairwise dissimilarities.

\subsection{Structural complexity and predictive robustness}
\label{subsec:complexity_robustness_results}


The theoretical analysis in Section~\ref{sec:complexity_stability}
identifies representation complexity and perturbation magnitude as two
distinct factors that may influence learning. We study these factors
separately. First, we examine associations between complexity of the Mapper representation and predictive behavior
across the parameter grid. For the empirical analysis, we use equal weights
\(a=b=c=1\), so that the reported complexity is
\[
TC(R)=|V|+|E|+\beta_1.
\] We then evaluate the sensitivity of a fixed
graph based predictor to controlled perturbations of its attributed
Mapper graph input.

\subsubsection{Complexity and predictive behavior}
\label{subsubsec:complexity_results}


\begin{figure}[h!]
    \centering
    \includegraphics[width=0.96\textwidth]
    {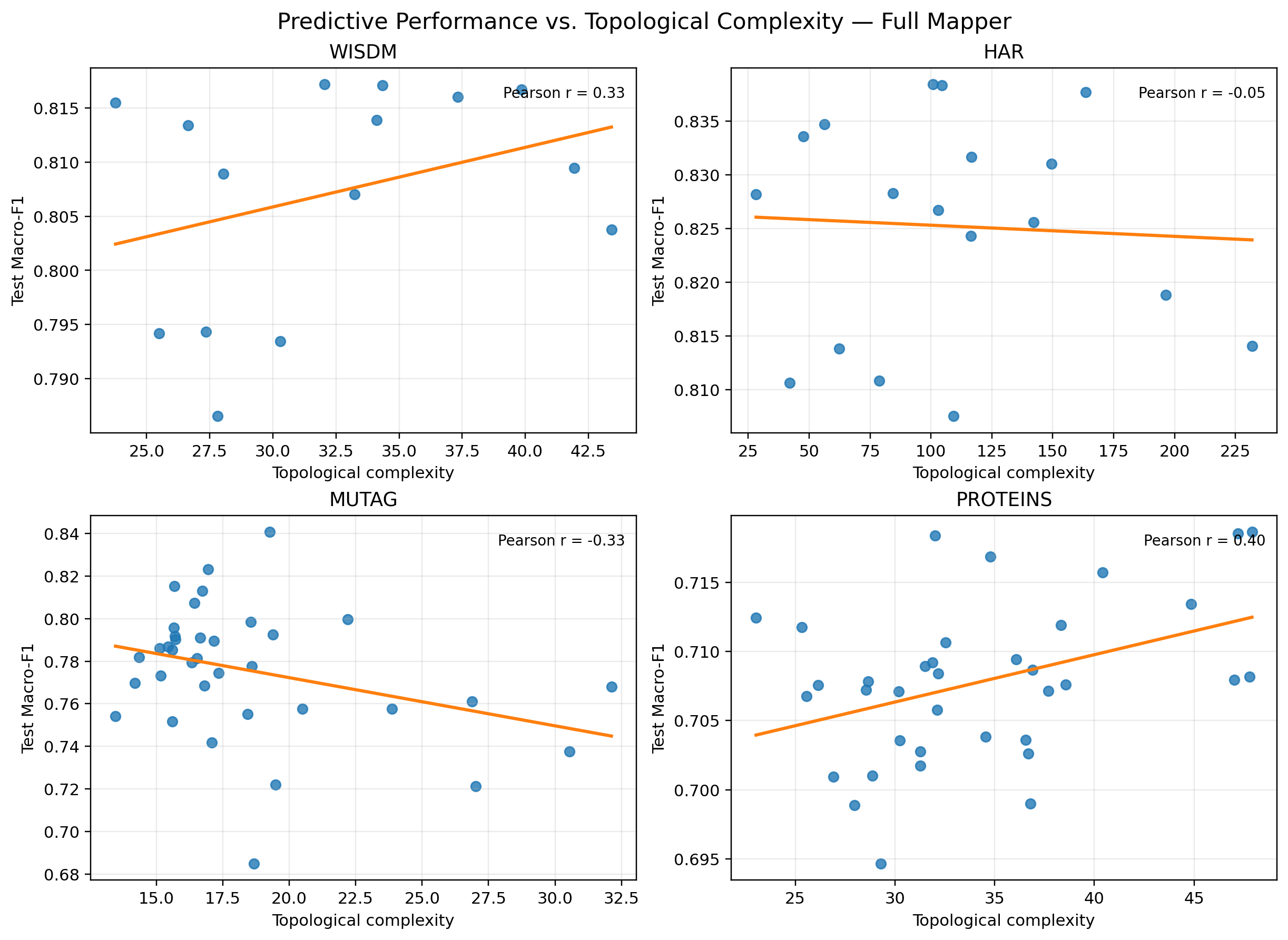}
    \caption{Relationship between topological complexity and test
    Macro-F1 across the four benchmark datasets.}
    \label{fig:complexity_performance_all}
\end{figure}

Figure~\ref{fig:complexity_performance_all} examines the relationship
between topological complexity and predictive performance for the
Original Mapper representation. Across the four datasets, the observed
correlations are weak to moderate and vary in direction:
$r=0.33$ for WISDM, $r=-0.05$ for HAR, $r=-0.33$ for MUTAG, and
$r=0.40$ for PROTEINS. Thus, increased topological complexity is not
associated with a consistent increase or decrease in test Macro-F1.

This result clarifies the role of the proposed complexity functional.
The measure is designed to summarize complementary aspects of Mapper
structure---local decomposition resolution, global topology, and
overlap interaction---rather than to serve as a direct predictor of
classification performance. The absence of a universal monotone
relationship indicates that structurally more complex Mapper
representations are not necessarily more predictive. Instead, the
usefulness of a particular level of complexity appears to depend on the
dataset and the resulting organization of the representation. 
This result motivates future work on more task aware notions of Mapper complexity.
These results also suggest that the present additive complexity
functional should be viewed as a first representation level summary;
alternative weighting schemes or task adaptive complexity measures may
provide a more refined characterization of the relationship between
Mapper structure and downstream learning.





\begin{figure}[t]
    \centering
    \includegraphics[width=0.85\textwidth]
    {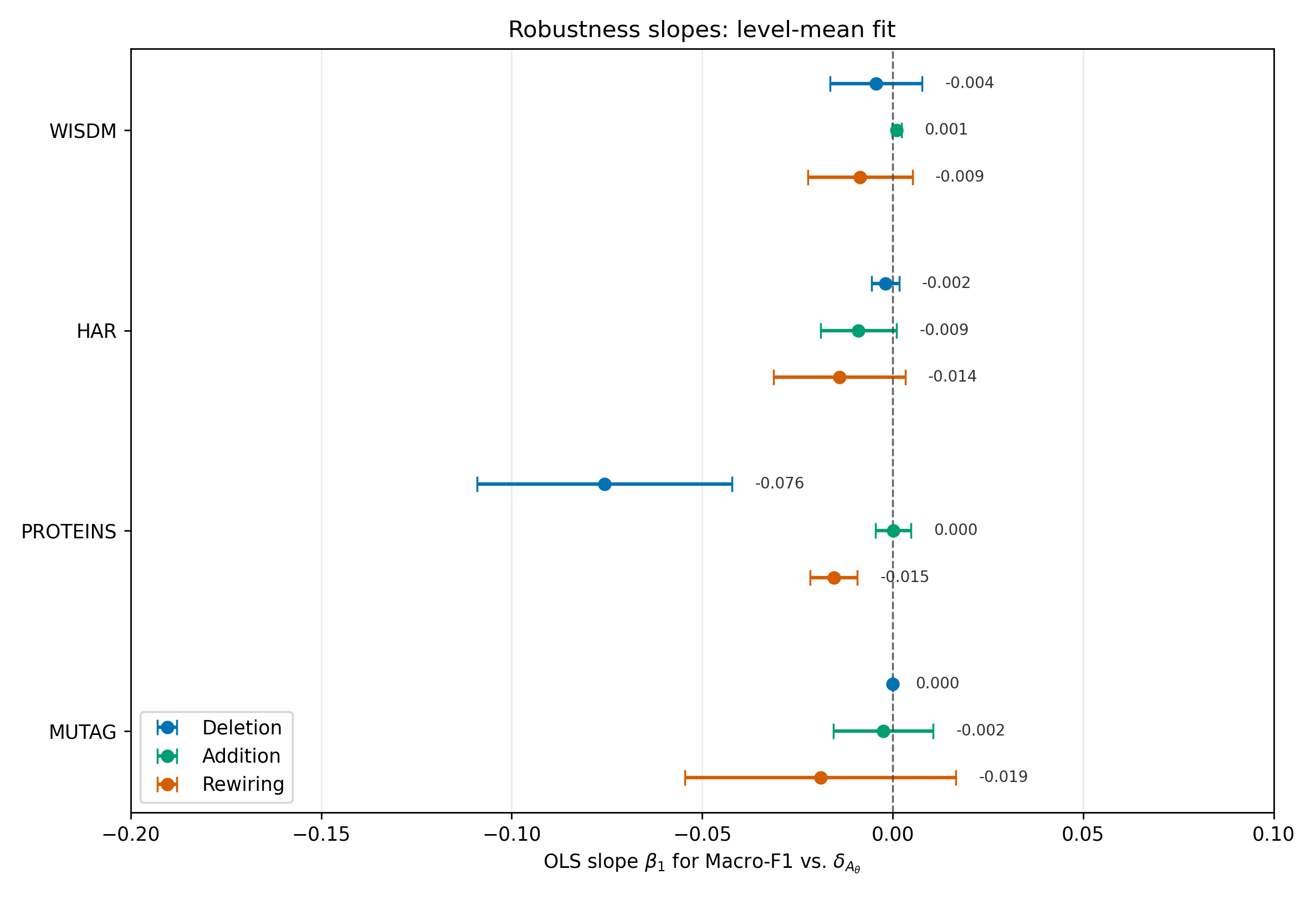}
   \caption{Estimated robustness slopes relating test Macro-F1 to perturbation level (\(\delta_{A_\theta}\)) for three edge perturbation mechanisms.}
    \label{fig:full_mapper_parameter_heatmaps}
\end{figure}

The parameter heatmaps in
Figure~\ref{fig:appendix_parameter_heatmaps} further demonstrate
that predictive performance depends jointly on cover resolution and
overlap, with no single parameter direction yielding uniformly superior
performance across all datasets. Nevertheless, a consistent qualitative
pattern emerges across most parameter settings: the Original Mapper
representation typically achieves the highest performance, the
Perturbed representation generally performs slightly worse, and the
No-edge representation is often the least predictive. Although this
ordering is not universal, it suggests that preserving increasingly
complete structural information is generally beneficial. At the same
time, the dependence on both cover resolution and overlap indicates
that structural richness alone does not determine predictive
performance, supporting the interpretation of topological complexity as
a descriptor of the representation rather than an objective to be
maximized.

\subsubsection{Robustness under attributed-graph perturbations}
\label{subsubsec:robustness_results}


We next evaluate the empirical counterpart of the prediction stability analysis in Section~\ref{sec:stability}. While the theoretical formulation allows perturbations to both the adjacency matrix and node-feature matrix, the present experiment focuses specifically on perturbations of relational structure. We therefore hold the node attributes fixed, so that $\widetilde X_\theta=X_\theta$,
and modify only the adjacency matrix \(A_\theta\). Starting from the original attributed Mapper graph, we generate perturbed representations using
three edge modification mechanisms. For each mechanism, we target
\(k\in\{0,2,5,7,10\}\) edge modifications. For \emph{edge deletion}, \(k\)
existing edges are sampled and removed; for \emph{edge addition}, \(k\)
previously absent edges are sampled and inserted; and for \emph{rewiring},
\(k\) existing edges are replaced by connections between previously
unconnected node pairs. When a graph contains too few eligible edges
or node pairs to realize the requested value of \(k\), the maximum
feasible number of modifications is applied. Consequently, the realized
perturbation count can be slightly smaller than the target value for
some datasets and Mapper constructions.

For each perturbed graph, we quantify the relative magnitude of the structural perturbation by
\[
\delta_{A_\theta}
=
\frac{\|A_\theta-\widetilde A_\theta\|_2}{\|A_\theta\|_2}
\]
and record the resulting test Macro-F1. The predictor parameters \(\phi\) are held fixed throughout, so that the experiment measures sensitivity to changes in Mapper connectivity rather than variability caused by changes in node attributes or model retraining.

Because random graph perturbations need not produce monotone changes in
classification performance, robustness is summarized using regression
slopes rather than requiring monotonic degradation. For each dataset,
Mapper construction, and perturbation type, we fit
\[
\mathrm{Macro\,F1}
=
\beta_0+\beta_1\delta_{A_\theta}+\varepsilon.
\]
A slope near zero indicates empirical insensitivity over the observed
perturbation range, whereas a negative slope indicates performance
degradation as the representation perturbation increases.

\begin{table}[t]
\centering
\caption{Estimated robustness slopes for test Macro-F1 as a function of
relative spectral perturbation. Confidence intervals are reported at the
95\% level.}
\label{tab:robustness_slopes}
\small
\begin{tabular}{llcc}
\toprule
\textbf{Dataset}
&
\textbf{Perturbation}
&
\(\widehat{\beta}_1\)
&
\textbf{95\% CI}
\\
\midrule
WISDM
& Deletion & \(-0.004\) & \([-0.016,\ 0.008]\) \\
& Addition & \( 0.001\) & \([-0.0002,\ 0.002]\) \\
& Rewiring & \(-0.009\) & \([-0.022,\ 0.005]\) \\
\midrule
HAR
& Deletion & \(-0.002\) & \([-0.006,\ 0.002]\) \\
& Addition & \(-0.009\) & \([-0.019,\ 0.001]\) \\
& Rewiring & \(-0.014\) & \([-0.031,\ 0.003]\) \\
\midrule
MUTAG
& Deletion & \( 0.000\) & \([ 0.000,\ 0.000]\) \\
& Addition & \(-0.003\) & \([-0.015,\ 0.01]\) \\
& Rewiring & \(-0.019\) & \([-0.054,\ 0.017]\) \\
\midrule
PROTEINS
& Deletion & \(-0.076\) & \([-0.109,\ -0.042]\) \\
& Addition & \( 0.0002\) & \([-0.004,\ 0.005]\) \\
& Rewiring & \(-0.015\) & \([-0.022,\ -0.009]\) \\
\bottomrule
\end{tabular}
\end{table}
The results indicate that robustness is representation- and
dataset-dependent. Several settings exhibit slopes close to zero,
showing that substantial changes in the attributed graph can occur
without a corresponding systematic decline in Macro-F1. Other
configurations, particularly selected deletion, addition, or rewiring
experiments, show more pronounced negative slopes. These differences
are consistent with the theoretical bound, which allows sensitivity to
depend on both the magnitude of the representation perturbation and the
properties of the trained message passing architecture.

The absence of monotone degradation in individual perturbation curves
does not contradict the stability analysis. Random edge modifications
may remove redundant connections, add benign edges, or occasionally
produce structures that improve a finite sample prediction. The
regression summaries characterize the average directional sensitivity
over repeated trials rather than asserting that every perturbation must
reduce performance.


\section{Discussion}

This work investigates how the representation chosen for a Mapper object
influences downstream learning. Rather than treating the Mapper graph as
a fixed computational output, we introduced a family of representations with different levels of retained structural information,
ranging from global graph summaries to attributed graphs with the
original nerve connectivity. Across four benchmark datasets, the
empirical results consistently demonstrate that representation choice
affects predictive performance. In particular, representations that
preserve the original connectivity generally achieve stronger predictive
accuracy than simplified variants, although the magnitude of the
improvement varies across datasets and parameter settings. This
indicates that relational information contained in the Mapper nerve
often contributes useful predictive signal beyond node-level
descriptors alone.

Beyond comparing predictive accuracy, the proposed representation
geometry provides a quantitative framework for studying relationships
between different Mapper constructions. By combining local, overlap, and
relational descriptors into a composite structural distance, the
resulting geometry organizes Mapper representations according to their
construction parameters and reveals meaningful similarities between
different parameter configurations. This perspective treats Mapper
representations as elements of a structured representation space rather
than isolated graph objects, providing a foundation for analyzing how
changes in Mapper construction alter the resulting representation before
any downstream learning is performed.

The stability analysis further complements this viewpoint. Both the
theoretical results and the empirical perturbation experiments indicate
that sensitivity is dataset- and perturbation-dependent, with several configurations exhibiting near-zero slopes and selected deletion/rewiring settings showing clearer degradation. At the same time, the experiments suggest that
no single notion of graph complexity or parameter choice universally
optimizes predictive performance. Instead, predictive behavior depends
on the interaction between cover resolution, overlap, and the structural
information retained in the representation, reinforcing the importance
of studying Mapper construction as part of the learning pipeline rather
than as a fixed preprocessing step.

Although this study focuses on Mapper, the underlying perspective is
considerably broader. Many topological and geometric data analysis
methods produce structured objects whose representations are simplified
before being used for statistical inference or machine learning. The
representation-centric framework developed here suggests that these
representation choices deserve systematic study in their own right.
Future work includes extending the hierarchy to richer topological
objects that preserve higher-order relationships, developing stronger
theoretical guarantees for additional learning algorithms beyond graph
neural networks, and applying the framework to scientific domains such
as biomedical imaging, genomics, neuroscience, and multimodal data.

\appendix

\section{Additional Experimental Results}
\label{app:additional_results}

This appendix provides supplementary experimental details that support
the analyses presented in Section~\ref{sec:exp}. Specifically, we
report the reference performance obtained without constructing Mapper
representations, summarize the categories of node attributes used for
each dataset, and present parameter-wise performance heat maps across
the complete Mapper parameter grid.

\subsection{Node attribute categories}
\label{app:node_features}

The node attributes associated with each refined pullback region are
constructed from domain-specific descriptors together with structural
properties of the corresponding Mapper graph. Rather than listing every
individual statistic, Table~\ref{tab:node_features} summarizes the
feature categories used for each dataset.

\begin{table}[h!]
\centering
\caption{Categories of node attributes associated with the refined
pullback regions for each benchmark dataset. All datasets additionally
include Mapper-node size and Mapper structural features.}
\label{tab:node_features}
\renewcommand{\arraystretch}{1.15}
\begin{tabular}{lcccc}
\toprule
\textbf{Feature category} &
\textbf{WISDM} &
\textbf{HAR} &
\textbf{MUTAG} &
\textbf{PROTEINS} \\
\midrule
Mapper node size                & \cmark & \cmark & \cmark & \cmark \\
Lens statistics                 & \cmark & \cmark &         &         \\
Acceleration statistics         & \cmark & \cmark &         &         \\
Gyroscope statistics            &         & \cmark &         &         \\
Temporal change statistics      & \cmark & \cmark &         &         \\
Node attribute statistics       &         &         & \cmark & \cmark \\
Label histogram                 &         &         & \cmark & \cmark \\
Graph topological summaries     &         &         & \cmark & \cmark \\
Induced subgraph descriptors    &         &         & \cmark & \cmark \\
Mapper structural features      & \cmark & \cmark & \cmark & \cmark \\
\bottomrule
\end{tabular}
\end{table}

\subsection{Learning and evaluation details}
\label{app:learning_details}

The attributed Mapper representations are evaluated using a two-layer
graph convolutional network (GCN) followed by global mean pooling and a
linear classifier. ReLU activation is applied after each graph
convolution. The learning configurations used for the four datasets are
summarized in Table~\ref{tab:gnn_hyperparameters}. The same
hyperparameters are used across representation variants within each
dataset group so that differences in predictive performance reflect
changes in the representation rather than model specific tuning.

For WISDM and HAR, ten-fold cross-validation is performed at the subject
level: subjects, rather than individual samples, are partitioned among
the folds. This prevents samples from the same subject from appearing
in both the training and test sets. For MUTAG and PROTEINS, stratified
ten-fold cross-validation is used. One model is trained for each fold,
giving ten training fits per experimental condition; the complete
cross-validation procedure is performed once.

The Global representation does not contain graph connectivity and is
therefore evaluated separately using a support vector machine (SVM).
The No-edge, Perturbed, and Original representations are evaluated
using the GCN architecture described above.

\begin{table}[h!]
\centering
\caption{Learning configurations used for the attributed Mapper
representations.}
\label{tab:gnn_hyperparameters}
\small
\begin{tabular}{lcc}
\toprule
\textbf{Setting} & \textbf{WISDM / HAR} & \textbf{MUTAG / PROTEINS} \\
\midrule
Architecture
& 2 GCNConv + linear
& 2 GCNConv + linear \\

Hidden dimension
& 64, 64
& 32, 32 \\

Dropout
& 0.25
& 0.50 \\

Optimizer
& Adam
& Adam \\

Learning rate
& 0.001
& 0.003 \\

Weight decay
& $10^{-4}$
& $5\times10^{-3}$ \\

Batch size
& 16
& 16 \\

Training budget
& 100 epochs
& Maximum 200 epochs \\

Early stopping
& None
& Patience 25 \\

Validation split
& -- 
& 15\% \\

Early-stopping criterion
& --
& Macro-F1 \\

Minimum improvement
& --
& $10^{-4}$ \\

\bottomrule
\end{tabular}
\end{table}

\subsection{Reference performance on the original data}
\label{app:raw_results}

Table~\ref{tab:raw_reference} reports the predictive performance
obtained directly from the original input data without constructing a
Mapper representation. These reference results provide context for the
difficulty of each benchmark dataset and are included solely for
comparison with the Mapper-based representations studied in this paper.

\begin{table}[h!]
\centering
\caption{
Reference predictive performance obtained directly from the original
feature representation without constructing a Mapper object. Values for
cross-validated experiments are mean Macro-F1
}
\label{tab:raw_reference}
\small
\begin{tabular}{p{2.2cm}p{5.2cm}cc}
\toprule
\textbf{Dataset}
&
\textbf{Original feature representation}
&
\textbf{Classifier}
&
\textbf{Macro-F1}
\\
\midrule

WISDM
&
109 statistical, temporal, correlation, and spectral features extracted
from each original 200-sample, three-axis accelerometer window.
&
SVM
&
\(0.885\)
\\[1ex]

HAR
&
Official 561-dimensional UCI HAR feature vectors computed from the
original accelerometer and gyroscope signals.
&
SVM
&
\(0.872\)
\\[1ex]

MUTAG
&
36-dimensional native-graph summary containing pooled statistics of the
original node attributes and basic graph-topology descriptors.
&
Random forest 
&
\(0.857\)
\\[1ex]

PROTEINS
&
24-dimensional native-graph summary containing pooled statistics of the
original node attributes and basic graph-topology descriptors.
&
SVM
&
\(0.756\)
\\

\bottomrule
\end{tabular}
\end{table}

\subsection{Parameter wise performance across the Mapper parameter grid}
\label{app:parameter_heatmaps}

To investigate the sensitivity of the proposed framework to Mapper
construction parameters, we evaluated every representation over the
common parameter grid described in
Section~\ref{subsec:experimental_design}. Each heat map reports the test Macro-F1
score for every combination of cover resolution and overlap
percentage. Together, these figures complement the averaged results in
the main text by illustrating how predictive performance varies across
different Mapper constructions.

\begin{figure}[t]
\centering
\begin{tabular}{c}
\includegraphics[width=.9\linewidth]{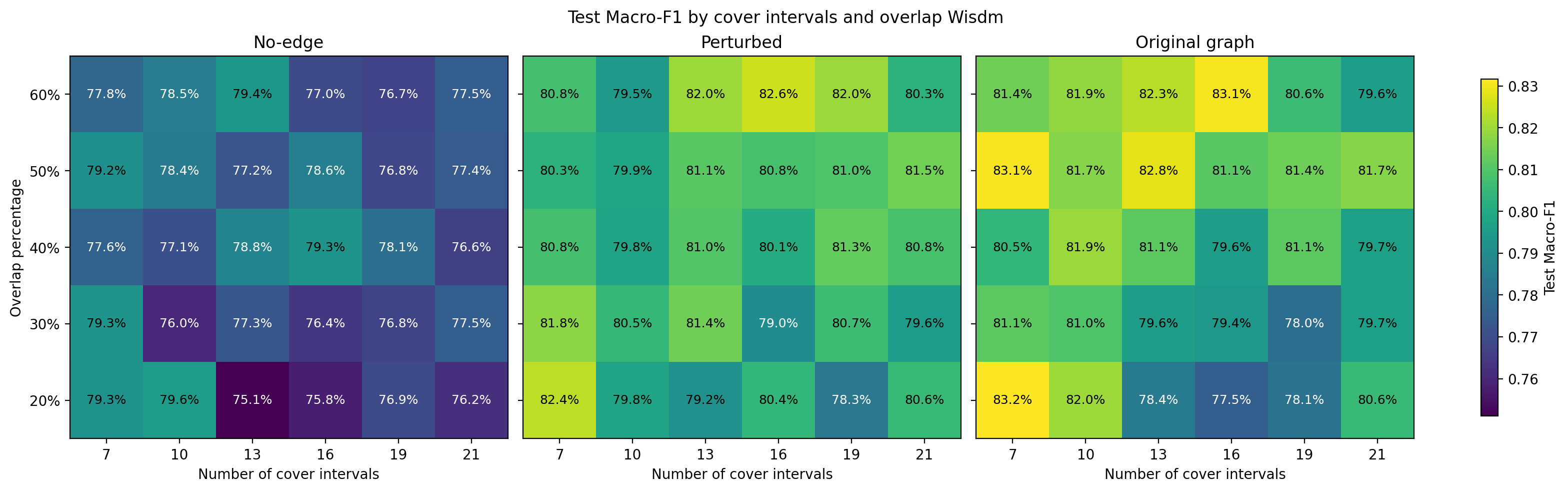} \\
(a) WISDM  \\[1em]
\includegraphics[width=.9\linewidth]{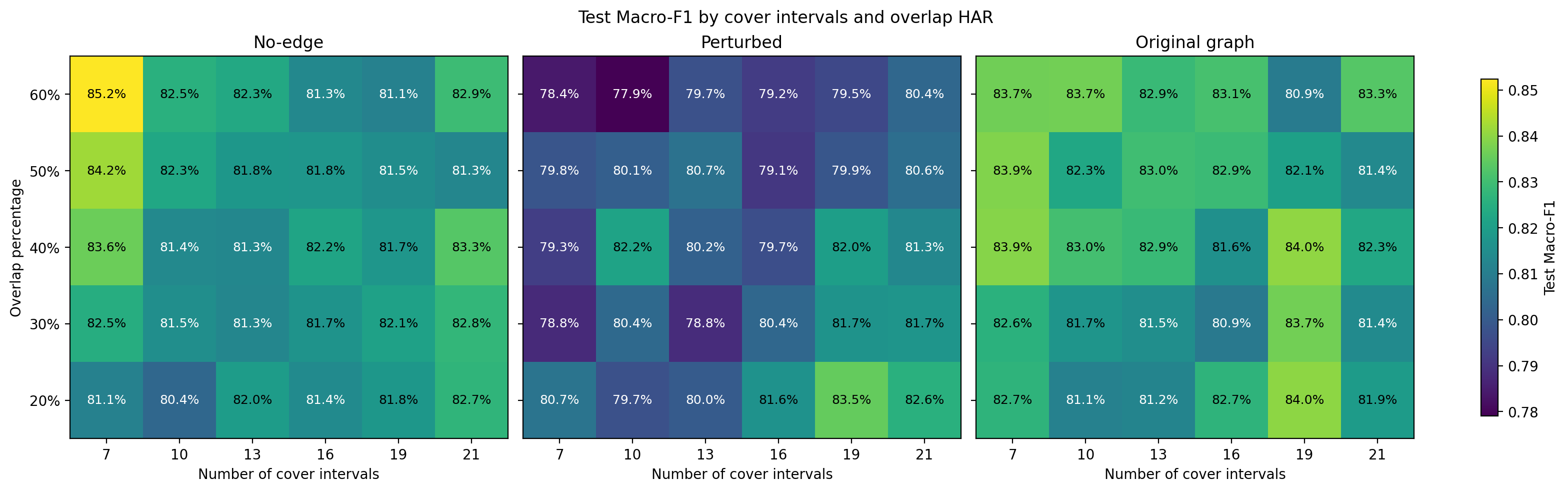} \\

(b) HAR \\

\includegraphics[width=.9\linewidth]{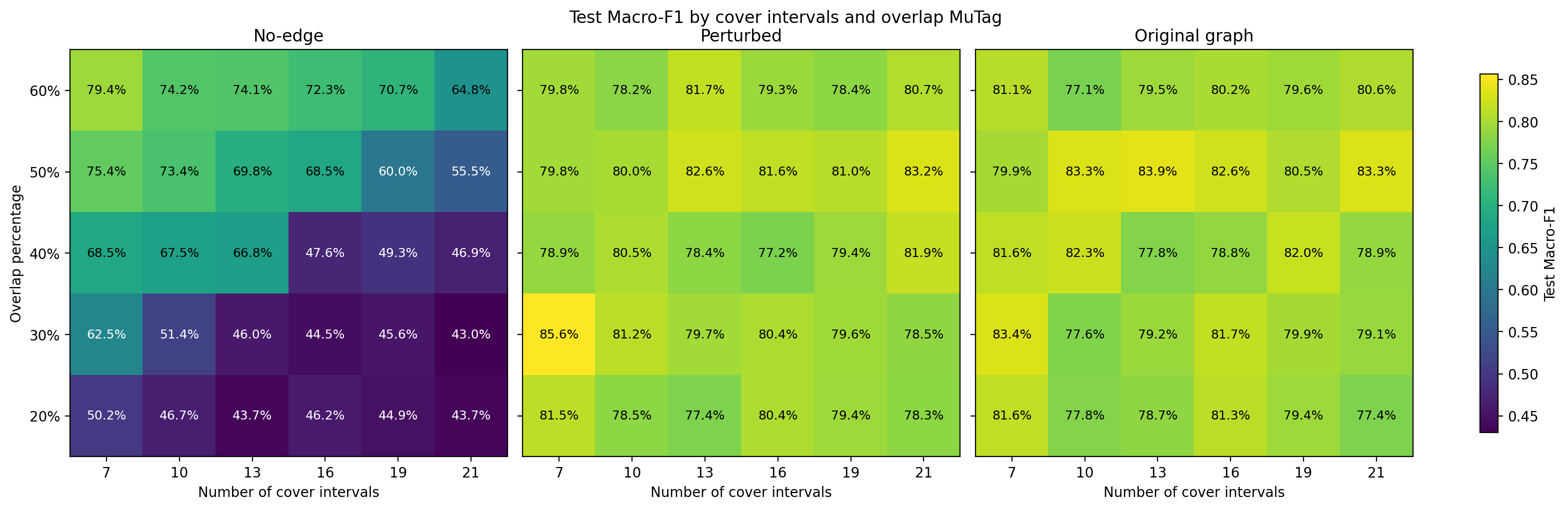} \\
(c) MUTAG \\ 
\includegraphics[width=.9\linewidth]{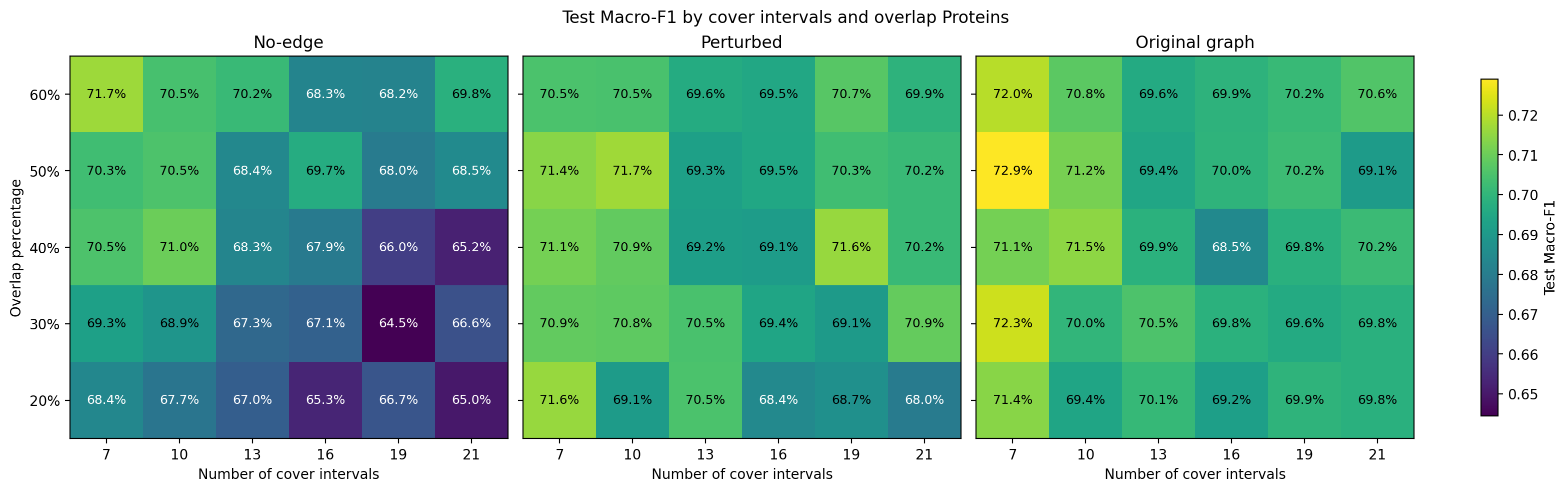} \\

(d) PROTEINS
\end{tabular}

\caption{Parameter wise predictive performance across the complete Mapper parameter grid for all four datasets.}
\label{fig:appendix_parameter_heatmaps}
\end{figure}

\bibliographystyle{plain}
\bibliography{reference}

@Preamble{ " \newcommand{\noop}[1]{} " }

@article{belkin2003laplacian,
  author  = {Belkin, Mikhail and Niyogi, Partha},
  title   = {Laplacian Eigenmaps for Dimensionality Reduction and Data Representation},
  journal = {Neural Computation},
  volume  = {15},
  number  = {6},
  pages   = {1373--1396},
  year    = {2003},
  doi     = {10.1162/089976603321780317}
}

@inproceedings{perozzi2014deepwalk,
  author    = {Perozzi, Bryan and Al-Rfou, Rami and Skiena, Steven},
  title     = {DeepWalk: Online Learning of Social Representations},
  booktitle = {Proceedings of the 20th ACM SIGKDD International Conference
               on Knowledge Discovery and Data Mining},
  pages     = {701--710},
  year      = {2014},
  doi       = {10.1145/2623330.2623732}
}

@inproceedings{grover2016node2vec,
  author    = {Grover, Aditya and Leskovec, Jure},
  title     = {node2vec: Scalable Feature Learning for Networks},
  booktitle = {Proceedings of the 22nd ACM SIGKDD International Conference
               on Knowledge Discovery and Data Mining},
  pages     = {855--864},
  year      = {2016},
  doi       = {10.1145/2939672.2939754}
}

@article{kipf2017gcn,
  title={Semi-Supervised Classification with Graph Convolutional Networks},
  author={Kipf, Thomas and Welling, Max},
  journal={ICLR},
  year={2017}
}

@article{hamilton2017representation,
  author  = {Hamilton, William L. and Ying, Rex and Leskovec, Jure},
  title   = {Representation Learning on Graphs: Methods and Applications},
  journal = {IEEE Data Engineering Bulletin},
  volume  = {40},
  number  = {3},
  pages   = {52--74},
  year    = {2017}
}

@book{edelsbrunner2010computational,
  title={Computational Topology: An Introduction},
  author={Edelsbrunner, Herbert and Harer, John},
  year={2010},
  publisher={AMS}
}

@inproceedings{singh2007topological,
  author    = {Singh, Gurjeet and M{\'e}moli, Facundo and Carlsson, Gunnar},
  title     = {Topological Methods for the Analysis of High Dimensional Data Sets
               and 3D Object Recognition},
  booktitle = {Proceedings of the 4th Eurographics Symposium on Point-Based Graphics},
  pages     = {91--100},
  year      = {2007},
  doi       = {10.2312/SPBG/SPBG07/091-100}
}

@article{lum2013extracting,
  author  = {Lum, P. Y. and Singh, G. and Lehman, A. and Ishkanov, T. and
             Vejdemo-Johansson, Mikael and Alagappan, M. and
             Carlsson, J. and Carlsson, Gunnar},
  title   = {Extracting Insights from the Shape of Complex Data Using Topology},
  journal = {Scientific Reports},
  volume  = {3},
  pages   = {1236},
  year    = {2013},
  doi     = {10.1038/srep01236}
}

@article{nicolau2011topology,
  author  = {Nicolau, Monica and Levine, Arnold J. and Carlsson, Gunnar},
  title   = {Topology-Based Data Analysis Identifies a Subgroup of Breast Cancers
             with a Unique Mutational Profile and Excellent Survival},
  journal = {Proceedings of the National Academy of Sciences},
  volume  = {108},
  number  = {17},
  pages   = {7265--7270},
  year    = {2011},
  doi     = {10.1073/pnas.1102826108}
}

@article{carriere2018statistical,
  author  = {Carri{\`e}re, Mathieu and Michel, Bertrand and Oudot, Steve},
  title   = {Statistical Analysis and Parameter Selection for Mapper},
  journal = {Journal of Machine Learning Research},
  volume  = {19},
  number  = {12},
  pages   = {1--39},
  year    = {2018},
  url     = {https://jmlr.org/papers/v19/17-291.html}
}

@article{Carlsson2009,
  author  = {Gunnar Carlsson},
  title   = {Topology and Data},
  journal = {Bulletin of the American Mathematical Society},
  volume  = {46},
  number  = {2},
  pages   = {255--308},
  year    = {2009},
  doi     = {10.1090/S0273-0979-09-01249-X}
}

@article{Wasserman2018,
  author  = {Wasserman, Larry},
  title   = {Topological Data Analysis},
  journal = {Annual Review of Statistics and Its Application},
  volume  = {5},
  number  = {1},
  pages   = {501--532},
  year    = {2018},
  doi     = {10.1146/annurev-statistics-031017-100045}
}

@article{ChazalMichel2021,
  author  = {Fr{\'e}d{\'e}ric Chazal and Bertrand Michel},
  title   = {An Introduction to Topological Data Analysis: Fundamental and Practical Aspects for Data Scientists},
  journal = {Frontiers in Artificial Intelligence},
  volume  = {4},
  pages   = {667963},
  year    = {2021},
  doi     = {10.3389/frai.2021.667963}
}

@article{Amezquita2023,
  author  = {Am{\'e}zquita EJ and Nasrin F and Storey KM and Yoshizawa M},
  title   = {Genomics data analysis via spectral shape and topology},
  journal = {PLoS One},
  volume  = {18},
  number  = {4},
  pages   = {e0284820},
  year    = {2023},
  doi     = {10.1371/journal.pone.0284820}
}

@inproceedings{reininghaus2015stable,
  title={A stable multi-scale kernel for topological machine learning},
  author={Reininghaus, Jan and Huber, Stefan and Bauer, Ulrich and Kwitt, Roland},
  booktitle={Proceedings of the IEEE Conference on Computer Vision and Pattern Recognition (CVPR)},
  pages={4741--4748},
  year={2015}
}

@article{adams2017persistence,
  author  = {Adams, Henry and Emerson, Tegan and Kirby, Michael and
             Neville, Rachel and Peterson, Chris and Shipman, Patrick and
             Chepushtanova, Sofya and Hanson, Eric and Motta, Francis and
             Ziegelmeier, Lori},
  title   = {Persistence Images: A Stable Vector Representation of Persistent Homology},
  journal = {Journal of Machine Learning Research},
  volume  = {18},
  number  = {8},
  pages   = {1--35},
  year    = {2017},
  url     = {https://www.jmlr.org/papers/v18/16-337.html}
}

@inproceedings{kusano2016persistence,
  author    = {Kusano, Genki and Hiraoka, Yasuaki and Fukumizu, Kenji},
  title     = {Persistence Weighted Gaussian Kernel for Topological Data Analysis},
  booktitle = {Proceedings of the 33rd International Conference on Machine Learning},
  pages     = {2004--2013},
  volume    = {48},
  series    = {Proceedings of Machine Learning Research},
  publisher = {PMLR},
  year      = {2016},
  url       = {https://proceedings.mlr.press/v48/kusano16.html}
}

@book{bakir2007predicting,
  title     = {Predicting Structured Data},
  editor    = {Bak{\i}r, G{\"o}khan H. and Hofmann, Thomas and
               Sch{\"o}lkopf, Bernhard and Smola, Alexander J. and
               Taskar, Ben and Vishwanathan, S. V. N.},
  publisher = {MIT Press},
  address   = {Cambridge, MA},
  year      = {2007}
}

@incollection{nowozin2011structured,
  title={Structured learning and prediction in computer vision},
  author={Nowozin, Sebastian and Lampert, Christoph H.},
  booktitle={Foundations and Trends in Computer Graphics and Vision},
  volume={6},
  number={3--4},
  pages={185--365},
  year={2011},
  publisher={Now Publishers}
}

@inproceedings{curry2023topologically,
  author    = {Curry, Justin and Mio, Washington and Needham, Tom and
               Okutan, Osman Berat and Russold, Florian},
  title     = {Topologically Attributed Graphs for Shape Discrimination},
  booktitle = {Proceedings of 2nd Annual Workshop on Topology, Algebra,
               and Geometry in Machine Learning (TAG-ML)},
  pages     = {87--101},
  volume    = {221},
  series    = {Proceedings of Machine Learning Research},
  publisher = {PMLR},
  year      = {2023},
  url       = {https://proceedings.mlr.press/v221/curry23a.html}
}

@inproceedings{comparingmappergraphs,
  author = {Youjia Zhou and Helen Jenne and Davis R. Brown and Madelyn R. Shapiro and Brett A. Jefferson and Cliff A. Joslyn and Gregory Henselman-Petrusek and Brenda Praggastis and Emilie Purvine and Bei Wang},
  title = {Comparing Mapper Graphs of Artificial Neuron Activations},
  booktitle = {Proceedings of the IEEE Workshop on Topological Data Analysis and Visualization (TopoInVis)},
  year = {2023},
  pages = {41--50},
  address = {Melbourne, Australia},
  publisher = {IEEE Computer Society},
  doi = {10.1109/TopoInVis60193.2023.00011},
  isbn = {979-8-3503-2964-3}
}

@article{deepgraphmapper,
  title={Deep Graph Mapper: Seeing Graphs Through the Neural Lens},
  author={Bodnar, Cristian and Cangea, Cătălina and Liò, Pietro},
  journal={Frontiers in Big Data},
  volume={4},
  pages={680535},
  year={2021}
}

@misc{cyranka2019mapper,
  author       = {Cyranka, Jacek and Georges, Alexander and Meyer, David},
  title        = {Mapper Based Classifier},
  year         = {2019},
  eprint       = {1910.08103},
  archivePrefix = {arXiv},
  primaryClass  = {cs.LG},
  url          = {https://arxiv.org/abs/1910.08103}
}

@misc{morris2020tudataset,
  author       = {Morris, Christopher and Kriege, Nils M. and Bause, Franka and
                  Kersting, Kristian and Mutzel, Petra and Neumann, Marion},
  title        = {{TUDataset}: A Collection of Benchmark Datasets for Learning with Graphs},
  year         = {2020},
  eprint       = {2007.08663},
  archivePrefix = {arXiv},
  primaryClass  = {cs.LG},
  url           = {https://arxiv.org/abs/2007.08663}
}

@book{scholkopf2002kernels,
  title={Learning with Kernels: Support Vector Machines, Regularization, Optimization, and Beyond},
  author={Sch{\"o}lkopf, Bernhard and Smola, Alexander J.},
  year={2002},
  publisher={MIT Press},
  address={Cambridge, MA}
}

@article{shervashidze2011weisfeiler,
  title={Weisfeiler-Lehman Graph Kernels},
  author={Shervashidze, Nino and Schweitzer, Pascal and van Leeuwen, Erik Jan and Mehlhorn, Kurt and Borgwardt, Karsten M.},
  journal={Journal of Machine Learning Research},
  volume={12},
  pages={2539--2561},
  year={2011}
}

@article{carriere2018mapper,
  author    = {Mathieu Carri\`ere and Steve Oudot},
  title     = {Structure and Stability of the {1D} Mapper},
  journal = {Foundations of Computational Mathematics},
  series    = {Proceedings of Machine Learning Research},
  volume    = {18},
  number    ={6},
  pages     = {1333–1396},
  year      = {2018}
}

@inproceedings{kwapisz2011activity,
  author    = {Jennifer R. Kwapisz and Gary M. Weiss and Samuel A. Moore},
  title     = {Activity Recognition using Cell Phone Accelerometers},
  booktitle = {Proceedings of the Fourth International Workshop on Knowledge Discovery from Sensor Data},
  pages     = {10--18},
  year      = {2011}
}

@inproceedings{anguita2013public,
  author    = {Davide Anguita and Alessandro Ghio and Luca Oneto and Xavier Parra and Jorge L. Reyes-Ortiz},
  title     = {A Public Domain Dataset for Human Activity Recognition Using Smartphones},
  booktitle = {ESANN},
  year      = {2013}
}

@article{debnath1991structure,
  author  = {Asim K. Debnath and Richard L. Lopez de Compadre and Gargi Debnath and Alan J. Shusterman and Corwin Hansch},
  title   = {Structure-Activity Relationship of Mutagenic Aromatic and Heteroaromatic Nitro Compounds},
  journal = {Journal of Medicinal Chemistry},
  volume  = {34},
  number  = {2},
  pages   = {786--797},
  year    = {1991}
}

@article{borgwardt2005protein,
  author  = {Borgwardt, Karsten M. and Ong, Cheng Soon and
             Sch{\"o}nauer, Stefan and Vishwanathan, S. V. N. and
             Smola, Alexander J. and Kriegel, Hans-Peter},
  title   = {Protein Function Prediction via Graph Kernels},
  journal = {Bioinformatics},
  volume  = {21},
  number  = {Suppl. 1},
  pages   = {i47--i56},
  year    = {2005},
  doi     = {10.1093/bioinformatics/bti1007}
}

@article{bengio2013representation,
  author  = {Yoshua Bengio and Aaron Courville and Pascal Vincent},
  title   = {Representation Learning: A Review and New Perspectives},
  journal = {IEEE Transactions on Pattern Analysis and Machine Intelligence},
  volume  = {35},
  number  = {8},
  pages   = {1798--1828},
  year    = {2013},
  doi     = {10.1109/TPAMI.2013.50}
}

@article{lecun2015deep,
  author  = {Yann LeCun and Yoshua Bengio and Geoffrey Hinton},
  title   = {Deep Learning},
  journal = {Nature},
  volume  = {521},
  number  = {7553},
  pages   = {436--444},
  year    = {2015},
  doi     = {10.1038/nature14539}
}

@book{goodfellow2016deep,
  author    = {Ian Goodfellow and Yoshua Bengio and Aaron Courville},
  title     = {Deep Learning},
  publisher = {MIT Press},
  address   = {Cambridge, MA},
  year      = {2016},
  isbn      = {9780262035613},
  url       = {https://www.deeplearningbook.org}
}

\appendix

\end{document}